\documentclass[10pt]{article}
\usepackage{verbatim}
\usepackage[usenames,dvipsnames,svgnames,table]{xcolor}
\usepackage[colorlinks=true,
            linkcolor=black,
            urlcolor=blue,
            citecolor=blue]{hyperref}

\usepackage[square, numbers]{natbib}

\usepackage{enumitem}
\usepackage{sectsty,amsmath, amsfonts, amsthm, amssymb, mathtools,dsfont,mathrsfs}
\usepackage[lined,linesnumbered,ruled,commentsnumbered]{algorithm2e} 
\usepackage{fullpage}
\usepackage{subcaption}
\usepackage{float}	
\usepackage{bbm}
\usepackage{etoolbox}
\usepackage{graphics}
\usepackage{svg}
\usepackage{booktabs}
\usepackage[titletoc]{appendix}
\usepackage{lineno}
\usepackage{ulem}
\DeclareMathOperator{\Var}{Var}
\DeclareMathOperator{\argmin}{arg\,min}

\numberwithin{equation}{section}

\title{Explainable post-training bias mitigation \\ with distribution-based fairness metrics}

\author{ Ryan Franks\thanks{Discover Financial Services, Riverwoods, IL, USA 
\\  \textit{\parindent=1em{\color{white}\hspace{1em} }The research was conducted while at Discover Financial Services (acquired by Capital One in May 2025).}} 
 \textsuperscript{,$\!\!\!$}
\thanks{first author, ryanfranks1996@gmail.com, ORCID:0000-0002-6007-1017}, \and
  Alexey Miroshnikov\textsuperscript{\specificthanks{1},}\thanks{principal investigator, amiroshn@terpmail.umd.edu, ORCID:0000-0003-2669-6336}, \and  Konstandinos Kotsiopoulos\textsuperscript{\specificthanks{1},}\thanks{kkotsiop@gmail.com, ORCID:0000-0003-2651-0087}  }

\date{}

\begin{document}

\makeatletter
\def\namedlabel#1#2{\begingroup
    #2%
    \def\@currentlabel{#2}
    \phantomsection\label{#1}\endgroup
}
\newcommand{\specificthanks}[1]{\@fnsymbol{#1}}
\newcommand{\lrarrow}{\mathrel{\mathpalette\lrarrow@\relax}}
\newcommand{\lrarrow@}[2]{%
  \vcenter{\hbox{\ooalign{%
    $\m@th#1\mkern6mu\rightarrow$\cr
    \noalign{\vskip1pt}
    $\m@th#1\leftarrow\mkern6mu$\cr
  }}}%
}
\makeatother
\newcommand{\tF}{\widetilde{F}}
\newcommand{\Lip}{{\rm Lip}}
\newcommand{\cT}{\mathcal{T}}
\newcommand{\loss}{\mathcal{L}}
\newcommand{\family}{\mathcal{F}}
\newcommand{\eloss}{\bar{\mathcal{L}}}
\newcommand{\E}{\mathbb{E}}
\newcommand{\1}{\mathbbm{1}}
\renewcommand{\P}{\mathbb{P}}
\newcommand{\Q}{\mathbb{Q}}
\newcommand{\T}{\top}
\newcommand{\trp}{^{\text{T}}}
\newcommand{\diag}{{\text{diag}}}
\newcommand{\Def}{\mathrm{Def}}
\newcommand{\f}[2]{\frac{#1}{#2}}
\newcommand{\gr}[1]{\nabla {#1}}
\newcommand{\Id}{{\bf I}}
\newcommand{\TT}{\mathcal{T}}
\newcommand{\nna}{{\rm num}}
\newcommand{\na}{ {\rm na} }
\newcommand{\EPE}{ {\rm EPE}}
\newcommand{\mat}{\hbox{Mat}}
\newcommand{\favdir}{\varsigma}
\def\del{\partial}
\def\torus{\mathbb{T}}
\def\Real{{\mathop{\hbox{\msym \char '122}}}}
\newcommand{\n}{\mbox{\boldmath $ \nu$}}
\newcommand{\cd}{{\cal D}}
\newcommand{\sma}{_{{ A}}}
\newcommand{\cE}{{\cal E}}
\newcommand{\sgn}{{\rm sgn}}
\newcommand{\dist}{{\mbox{dist}}}
\newcommand{\cS}{{\cal S}}
\newcommand{\cA}{{\cal A}}
\newcommand{\cC}{{\cal C}}
\newcommand{\cP}{{\cal P}}
\newcommand{\cB}{{\cal B}}
\newcommand{\cL}{{\cal L}}
\newcommand{\cG}{{\cal G}}
\newcommand{\cF}{{\cal F}}
\newcommand{\cI}{{\cal I}}
\newcommand{\cR}{{\cal R}}
\newcommand{\cK}{{\cal K}}
\newcommand{\hj}{^{J}}
\newcommand{\m}{\mbox{\boldmath $ \mu$}}
\newcommand{\nutx}{\nu_{t,x}}
\newcommand{\mutx}{\mu_{t,x}}
\newcommand{\hjm}{^{J-1}}
\newcommand{\cu}{{\cal U}}
\newcommand{\tn}{{\tilde\|}}
\newcommand{\rbar}{\overline{r}}
\newcommand{\dxdt}{\; dxdt}
\newcommand{\dx}{\; dx}
\newcommand{\oeps}{\overline{\varepsilon}}
\newcommand{\cgl}{\hbox{Lie\,}{\cal G}}
\newcommand{\ttd}{{\tt d}}
\newcommand{\ttdel}{{\tt \delta}}
\newcommand{\ns}{\nabla_*}
\newcommand{\csl}{{\cal SL}}
\newcommand{\spsi}{_{{{\Psi}}}}
\newcommand{\kr}{\hbox{Ker}}
\newcommand{\Tcal}{\mathcal{T}}
\newcommand{\Pcal}{\mathcal{P}}
\newcommand{\qinfo}{\stackrel{\circ}{Q}_{\infty}}
\newcommand{\eq}[1]{\begin{equation*}#1\end{equation*}}
\newcommand{\diagmtx}[3]{\begin{bmatrix} #1 \; \quad \; \quad\\\quad \; #2 \; \quad \\\quad \; \quad \; #3\end{bmatrix}}

\newcommand{\B}{\mathcal{B}}
\newcommand{\cof}{\mathrm{cof\, }}
\newcommand{\RR}{\mathbb{R}}
\newcommand{\PP}{\mathbb{P}}
\newcommand{\eps}{\varepsilon}
\newcommand{\Chi}{\mathcal{X}}
\newcommand{\D}{\mathrm{D}}
\newcommand{\cD}{\mathcal{D}}
\newcommand{\pd}[2]{\frac{\partial #1}{\partial #2}}
\newcommand{\BBR}[1]{\left( #1 \right)}
\newcommand{\BBS}[1]{\left[ #1 \right]}
\newcommand{\BBF}[1]{\left\{ #1 \right\}}
\newcommand{\BBN}[1]{\left\| #1 \right\|}
\newcommand{\BBA}[1]{\left | #1 \right |}
\newcommand{\mdd}[1]{M^{#1\times #1}}
\newcommand{\mddplus}[1]{M^{#1\times #1}_+}
\newcommand{\mddplusb}[1]{{\bar{M}}^{#1\times #1}_+}
\newcommand{\PHD}{\frac{\partial{\varPhi^A}}{\partial{F_{i\alpha}}}}
\newcommand{\PSD}{\frac{\partial{\Xi^A}}{\partial{F_{i\alpha}}}}

\newcommand{\rbcBias}{\widehat{B}ias}
\newcommand{\Bias}{\text{\rm Bias}}
\newcommand{\MSE}{\mbox{\rm MSE}}
\newcommand{\MISE}{\mbox{\rm MISE}}
\newcommand{\VAR}{\mathbb{V}}
\newcommand{\opt}{\mbox{opt}}
\newcommand{\Rhat}{\widehat{R}}
\newcommand{\fhat}{\hat{f}}
\newcommand{\fbar}{\bar{f}}
\newcommand{\ftilde}{\tilde{f}}
\newcommand{\Ytilde}{\widetilde{Y}}
\newcommand{\phat}{\widehat{p}}
\newcommand{\yhat}{\widehat{y}}
\newcommand{\Yhat}{\widehat{Y}}
\newcommand{\signbias}{ \widetilde{bias} }
\newcommand{\bias}{ bias }
\newcommand{\perf}{ \text{\it perf}\, }
\newcommand{\modbias}{ \Bias }

\newcommand{\PDP}{ P\!D\!P }
\newcommand{\SHAP}{ S\!H\!A\!P}
\newcommand{\QSHAPP}{ S\!H\!A\!P^{\mathcal{P}}}
\newcommand{\ME}{ \text{\tiny \it ME}}
\newcommand{\CE}{ \text{\tiny \it CE}}
\newcommand{\ICE}{ \text{\tiny \it ICE}}
\newcommand{\IBE}{ \text{\tiny \it IBE}}
\newcommand{\vpdp}{v^{\ME}}
\newcommand{\vce}{v^{\CE}}
\newcommand{\vme}{\vpdp}
\newcommand{\cvpdp}{\tilde{v}^{\PDP}}
\newcommand{\cvce}{\tilde{v}^{\CE}}
\newcommand{\vpdpP}{v^{\ME,\mathcal{P}}}
\newcommand{\vceP}{v^{\CE,\mathcal{P}}}
\newcommand{\A}{\mathcal{A}}

\newcommand{\txq}[1]{\quad \mbox{#1} \quad}
\newcommand{\txc}[1]{\; \mbox{#1} \;}
\newcommand{\txs}[1]{\: \mbox{#1} \:}
\newcommand{\tx}{\mbox}

\def\clap#1{\hbox to 0pt{\hss#1\hss}}
\renewcommand{\dddot}[1]{\overset{\small\mathrm{\raisebox{-0.15ex}{\clap{$\displaystyle\hspace{0.1em}.\hspace{-0.1em}.\hspace{-0.1em}.$}}}}{#1}}
\mathchardef\mhyphen="2D

\newtheorem{theorem}{Theorem}[section]
\newtheorem{definition}{Definition}[section]
\newtheorem{lemma}{Lemma}[section]
\newtheorem{remark}{Remark}[section]
\newtheorem{example}{Example}[section]
\newtheorem{proposition}{Proposition}[section]
\newtheorem{corollary}[theorem]{Corollary}
\newtheorem*{convention*}{Convention}
\newtheorem*{corollary*}{Corollary}

\mathchardef\mhyphen="2D
\newcommand{\cbias}{c\mhyphen bias}
\newcommand{\signcbias}{ \widetilde{c\mhyphen bias} }

\maketitle

\begin{abstract}
We develop a novel bias mitigation framework with distribution-based fairness constraints suitable for producing demographically blind and explainable machine-learning models across a wide range of fairness levels. This is accomplished through post-processing, allowing fairer models to be generated efficiently without retraining the underlying model. Our framework, which is based on stochastic gradient descent, can be applied to a wide range of model types, with a particular emphasis on the post-processing of gradient-boosted decision trees. Additionally, we design a broad family of global fairness metrics, along with differentiable and consistent estimators compatible with our framework, building on previous work. We empirically test our methodology on a variety of datasets and compare it with alternative post-processing approaches, including Bayesian search, optimal transport projection, and direct neural network training.
\end{abstract}

\smallskip

\textbf{Keywords.} ML fairness, ML interpretability, bias mitigation, post-processing.

\vspace{5pt}

\textbf{AMS subject classification.}  49Q22, 65K10, 91A12, 68T01.

\section{Introduction}

Machine learning (ML) techniques have become ubiquitous in the financial industry due to their powerful predictive performance.  However, ML model outputs may lead to certain types of unintended bias, which are measures of unfairness that impact protected sub-populations.

Predictive models, and strategies that rely on such models, are subject to laws and regulations that ensure fairness. For instance, financial institutions (FIs) in the U.S. that are in the business of extending credit to applicants are subject to the Equal Credit Opportunity Act (ECOA) \citep{ECOA} and the Fair Housing Act (FHA) \citep{FHA1974}, which prohibit discrimination in credit offerings and housing transactions. The protected classes identified in the laws, including race, gender, age (subject to very limited exceptions), ethnicity, national origin, and marital status, cannot be used as attributes in lending decisions.

While direct use of protected attributes is prohibited under ECOA when training any ML model, other attributes can still act as their ``proxies'', which may potentially lead to discriminatory outcomes. For this reason, it is crucial for FIs  to evaluate predictive models for potential bias without sacrificing their high predictive performance.

In this work, we propose a framework that addresses the problem of ML model bias in the financial industry setting under regulatory constraints. In such contexts, bias mitigation methods must
satisfy multiple criteria: they must accommodate diverse predictive models, avoid direct use of the protected
attribute, yield models that maintain explainability, and be computationally efficient to enable the construction
of models along a fairness–performance frontier. Furthermore, it is often crucial to employ global bias metrics
that evaluate fairness over the entire score distribution, since the decision thresholds of classifiers are not
typically known at the training stage and may differ across downstream applications. To accomplish this, we employ global bias metrics such as those used in \cite{Kwegyir-Aggrey2023, Chzhen2020, Chzhen2022} and gradient descent based optimization techniques as in \cite{Vogel2021, Jiang2020Jul}.

There is a comprehensive body of research on fairness metrics and bias mitigation. The bias mitigation approaches discussed in the survey paper \citep{mehrabi2022} depend on the operational flow of model development processes and fall into one of three categories: pre-processing methods, in-processing methods, and post-processing methods. Pre-processing methods modify datasets before model development to reduce the bias in trained models. In-processing methods modify the model development procedure itself. Finally, post-processing methods adjust already-trained models to be less biased. Each category presents trade-offs that affect their suitability for real-world applications, particularly in regulated domains such as finance.

Pre-processing methods may reduce the strength of relationships between the features and protected class as in \citep{delBarrio, Feldman2015}, which apply optimal transport methods to adjust features. Alternatively, they may re-weight the importance of observations as in \citep{calders2009, Jiang2020}, or adjust the dependent variable  \citep{Kamiran2009}. By employing these techniques, one can reduce the bias of any model trained on the modified dataset.

In-processing methods modify the model selection procedure or adjust the model training algorithm to reduce bias. For example, \citep{Perrone2020} introduces bias as a consideration when selecting model hyperparameters using Bayesian search. For tree-based models, \citep{Kamiran2010} modifies the splitting criteria and pruning procedures used during training to account for bias. For neural networks, \citep{Vogel2021} alters the loss function with a bias penalization based on receiver operating characteristic curves. Similarly, \citep{Jiang2020Jul} proposes training logistic regression models using a bias penalization based on the 1-Wasserstein barycenter \cite{AguehCarlier2011,Brizzi2025} of subpopulation score distributions.

Post-processing methods either reduce the bias in classifiers derived from a given model, as in \cite{Hardt2015,Dwork2012}, or reduce the model bias according to a global metric (e.g., the Wasserstein bias \cite{Miroshnikov2020}). To this end, \cite{Kwegyir-Aggrey2023,Chzhen2020,Chzhen2022} adjust score subpopulation distributions via optimal transport, while \cite{Miroshnikov2020b} optimizes a bias-penalized loss through Bayesian search  over a family of models constructed by scaling inputs to a trained model. Lastly, \cite{pangia2024} proposes a novel re-processing approach for tree ensembles, where boosting is applied first and the resulting decision trees are then adjusted to satisfy a fairness constraint based on information value \cite{siddiqi2006credit}.

While these approaches are often useful, their application in financial services may be limited by the aforementioned practical and regulatory considerations. To elaborate, we seek a bias mitigation strategy that meets the following criteria:
\begin{itemize}
\item [$(i)$] \textsc{Global bias metrics}. Binary decisions are made by thresholding a model score by a cut-off value unknown at the model development stage. Thus, the methodology should support a range of metrics that evaluate classifier bias across decision thresholds of interest, such as the metrics in \cite{Vogel2021,Jiang2020Jul, Miroshnikov2020, Becker2024}.
\item  [$(ii)$] \textsc{Model flexibility.} The methodology should be  applicable to different types of models, such as generalized linear models, neural networks, tree ensembles, etc., to accommodate a range of tasks.
\item [$(iii)$] \textsc{Demographic-blindness.} Fairer models must have no explicit dependence on the protected attribute. Its use for inference may be prohibited by law, and furthermore, collecting information on it may be practically infeasible, except for proxy information such as in \cite{Elliot2009} for validation purposes.

\item [$(iv)$]  \textsc{Explainability.} Fairer models should be explainable, as regulations in FIs require applicants to be informed of factors leading to adverse credit decisions\footnote{See \citep{Hall2021} for further discussion of regulatory constraints impacting FIs, and Section  \ref{sec::distfairness} for further details on explainability.}. That is, we require that an existing explainability algorithm (e.g., \cite{Ribeiro2016,Strumbelj2014,LundbergTreeSHAP,LundbergLee,LundbergIntervTreeSHAP,filom2024,yang2021gami})  can be applied to or extended for the fairer model at a reasonable computational cost.

\item [$(v)$] \textsc{Efficient frontiers.} The method must be computationally fast to allow for the construction of a range of predictive models with different bias values, enabling the selection of a model with an appropriate bias-performance trade-off at a later stage.
\end{itemize}

For many existing in-processing bias mitigation approaches, these criteria pose challenges. Some notable methods, such as \cite{Jiang2020Jul,Vogel2021}, are appealing in light of their use of distribution-based bias metrics and gradient-based optimization. Specifically, \cite{Jiang2020Jul} trains logistic regression models with a loss penalized by the 1-Wasserstein barycenter of subpopulation score distributions, while \cite{Vogel2021} introduce a custom loss integrating a ROC-based fairness metric into neural network training. However, they limit model flexibility $(ii)$ and explainability $(iv)$, and they preclude the use of tree ensembles, which are often preferred for tabular data \citep{schwartz-ziv2022}. Similarly, \cite{Ravichandran2020} proposes an extension of XGBoost that integrates fairness constraints into its gradient and Hessian computations. These constraints must be expressed as expected losses, making them incompatible with most integral probability metrics such as Wasserstein that lack per-sample loss formulations.

Some in-processing methods avoid reliance on specific model architectures, improving flexibility. For example, \cite{Perrone2020} introduces a model-agnostic approach that incorporates fairness into hyperparameter tuning via Bayesian optimization. While this can satisfy criteria $(i)$–$(iv)$, the reliance on Bayesian search limits optimization power \citep{frazier2018}, posing challenges for criterion $(v)$.

Post-processing methods such as \cite{Miroshnikov2020b} also rely on Bayesian optimization and therefore have similar limitations in computational efficiency. Other post-processing methods, such as \cite{Jiang2020Jul} and \cite{Kwegyir-Aggrey2023}, employ optimal transport to align score distributions across subpopulations. These methods address criteria $(i)$, $(ii)$, and $(v)$, but produce models that explicitly depend on the protected attribute -- an exception being \cite{miroshnikovpatent}, where the dependence is removed. Furthermore, these approaches transform the trained model in ways that hinder explainability.

Overall, existing pre-processing and in-processing approaches often require costly retraining to obtain fairer models across different bias levels (i.e., the efficient bias-performance frontier), which may be computationally impractical for large datasets. Additionally, such approaches reduce flexibility, as the fairness metric is embedded into the training process. This makes later adaptation to regulatory or institutional changes difficult. Conducting bias mitigation in the post-processing stage, by contrast, allows practitioners to select or refine the fairness metric after model training, providing flexibility and adaptability.

Our novel post-processing framework for bias mitigation, motivated by \cite{Jiang2020Jul,Vogel2021}, addresses all criteria $(i)$–$(v)$. We introduce a broad family of global bias metrics and develop differentiable estimators suitable for optimization via gradient descent. We then construct parameterized families of post-processed models that can be adapted from any trained, demographically blind, and explainable model architecture. These model families are designed to remain explainable and allow efficient gradient-based optimization. Finally, we construct bias-performance efficient frontiers of these model families by optimizing under fairness constraints, expressed through our proposed global bias estimators. This optimization is carried out using stochastic gradient descent, in the same manner that one may train a neural network. 

Unlike in-processing approaches, which must embed the fairness metric into the training process, our method performs fairness optimization over a family of correctors (encoders), enabling fast and flexible construction of fair models. This is particularly well suited for tree-based models, such as gradient-boosted decision trees, where fairness constraints are hard to incorporate during training. We elaborate on the most salient technical details of this strategy below.

To ground the discussion, we consider the joint distribution $(X, Y, G)$, where $X=(X_1,\dots,X_n)$ is a vector of features, $Y\in \RR$ is the response variable, and $G \in \{0,1,\dots,K-1\}$ represents the protected attribute. Addressing criterion $(i)$, we propose a broad class of distribution-based bias metrics that quantify disparities between the distributions of model scores across protected groups. When $G \in \{0,1\}$ is binary, the metric takes the form
\begin{equation}\label{intro:bias}
\cB(f|X,G) := \int c\bigl(F_{f(X)|G=0}(t), F_{f(X)|G=1}(t)\bigr) \, \mu(dt),
\end{equation}
where $c(\cdot,\cdot)$ is a cost function, $F_{f(X)|G=k}$ is the cumulative distribution function of $f(X) | G=k$, and $\mu$ is a probability measure signifying the importance of the classifier associated with threshold $t \in \mathbb{R}$ which may in principle, depend on $(f, P_G, P_{X|G})$ as in \cite{Vogel2021,Becker2024}. For a raw probability score, $f$ in \eqref{intro:bias} is replaced with $\operatorname{logit}(f)$. This formulation encompasses many metrics including the 1-Wasserstein metric, the energy distance \cite{Szekely1989}, and others \cite{Becker2024, Vogel2021}. Moreover, it generalizes naturally to non-binary protected attributes as in \cite{Jiang2020Jul, Miroshnikov2020}. We design estimators for these metrics via relaxation methods inspired by \cite{Vogel2021} and perform an asymptotic analysis to justify their use in numerical optimization; see Section \ref{sec::biasapprox}.

We can now design generic, model-flexible methods to produce families of post-processed models. Given a trained regressor or raw probability score model  $f_*$, we select a vector  $w=(1,w_1,\dots, w_m)(x;f_*)$ of weight functions (or encoders) and construct the corresponding family of models:
\begin{equation}\label{intro:family}
\cF(f_*;w) := \big\{f_{\theta}:  f_{\theta}(x;f_*):=f_*(x) - \theta \cdot w(x;f_*), \,\, \theta \in \RR^{m+1} \big\},
\end{equation}
where $\theta \in \RR^{m+1}$ is learnable, and $w$ may generally depend on the underlying model representation; see Section \ref{sec::outputperturbation}. In accordance with criteria $(iii)$ and $(iv)$, we typically require that both $f_*$ and $w$ be demographically blind and explainable.

Finally, we tackle criterion $(v)$ and seek models in $\cF(f_*,w)$ whose bias-performance trade-off is optimal -- that is, the least biased among similarly performing models. To construct the efficient frontier of $\cF(f_*;w)$, adapting the approaches in \cite{Jiang2020Jul,Vogel2021}, we solve a minimization problem with a fairness penalization \citep{Karush,KuhnTucker}: 
\[
\theta^*(\omega) := {\rm argmin}_{\theta} \{\cL(f_\theta) + \omega \B(f_{\theta}|X,G)\},
\]
where $\cL$ is a given loss function and $\omega\geq 0$ is a bias penalization coefficient.

Crucially, the above minimization problem is linear in $w$. Unlike the post-processing approach in \cite{Miroshnikov2020b}, this setup circumvents the lack of differentiability of the trained model, enabling the use of gradient-based methods as opposed to Bayesian search even when $f_*$ is discontinuous (e.g., tree-based ensembles). Consequently, we may efficiently post-process any model while optimizing a high-dimensional parameter space via stochastic gradient descent.

Furthermore, given an explainer map $(x,f,X) \mapsto E(x;f,X) \in \RR^n$, assumed to be linear in $f$, the explanation of any model in \eqref{intro:family} can be expressed in terms of those of the trained model and the encoders\footnote{In some cases, our method is compatible with explanations that are not linear in $f$ such as path-dependent TreeSHAP \cite{LundbergTreeSHAP}.}. Thus, the explanations for any model in the family  can be efficiently reconstructed  for an entire dataset.

Clearly, the choice of encoder family is key to ensuring both demographic blindness and explainability in post-processed models. We propose three such families -- based on additive models, weak learners (for tree ensembles), and model explanations -- that yield interpretable and flexible corrections; see Section~\ref{sec::methods_perturb}.

Our approach enables fast construction of demographically blind, explainable models with favorable bias-performance trade-offs. We evaluate it on both synthetic and real-world datasets \citep{adult_2, bank_marketing_222, compas}, comparing against the post-processing method of \cite{Miroshnikov2020b} and an in-processing neural network approach adapted from \cite{Vogel2021}. We find that our post-processing method yields strong bias-performance frontiers. We also examine how dataset characteristics influence outcomes and suggest strategies to mitigate overfitting.

\vspace{5pt}

\noindent{\bf Structure of the paper.}  In Section \ref{sec::preliminaries}, we introduce the requisite notation and fairness criteria for describing the bias mitigation problem, approaches to defining model bias, and an overview of model explainability. In Section \ref{sec::biasapprox}, we provide differentiable estimators for various bias metrics as well as asymptotic analysis. In Section \ref{sec::methods_perturb}, we introduce post-processing methods for explainable bias mitigation using stochastic gradient descent. In Section \ref{sec::experiments}, we systematically compare these methods on synthetic and real-world datasets. In the appendix, we provide various auxiliary lemmas and theorems as well as additional numerical experiments.

\section{Preliminaries}
\label{sec::preliminaries}

\subsection{Notation and hypotheses}\label{sub:notation_and_hypotheses}

In this work, we investigate post-training bias mitigation methods that address distribution-based fairness constraints while preserving model explainability. We are given a joint distribution triple $(X, G, Y)$ composed of predictors $X=(X_1,X_2,\dots,X_n)$, a response variable $Y$, and a demographic attribute $G\in \{0,1,\dots, K-1\}=:\mathcal{G}$ which reflects the subgroups that we desire to treat fairly. We assume that all random variables are defined on the common probability space $(\Omega,\mathcal{F},\PP)$, where $\Omega$ is a sample space, $\PP$ a probability measure, and $\mathcal{F}$ a $\sigma$-algebra of sets. Finally, the collection of Borel functions on $\RR^n$ is denoted by  $\mathcal{C}_{\mathcal{B}(\RR^n)}$.

With this context, the bias mitigation problem seeks to find Borel models $f(x)$ that typically approximate the regressor $\E[Y | X=x]$ or classification score $\PP(Y = 1 | X=x)$ (if $Y\in\{0,1\}$) that are less biased according to some model bias definition. Typically these definitions require one to determine the key fairness criteria for the business process employing $f(x)$, how deviations from these criteria will be measured, and finally how these deviations relate to the model $f(x)$ itself. Below, we review this process to properly contextualize model-level bias metrics of interest.

Given a model $f$ and features $X$, we set $Z:=f(X)$ and denote model subpopulations as $Z_k:=f(X)|G=k$, $k \in \mathcal{G}$. The subpopulation cumulative distribution function (CDF) of  $Z_k$ is denoted by 
\[F_k(t):=F_{f(X)|G=k}(t)=\PP(f(X)\leq t|G=k),\]
and the corresponding generalized inverse (or quantile function) $F_k^{[-1]}$ is defined by 
\[
F_k^{[-1]}(p):=F_{f(X)|G=k}^{[-1]}(p)=\inf_{x \in \RR }\big\{ p \leq F_k(x) \big\},\]
for each $k \in \cG$. Finally, given a threshold $t \in \RR$, the corresponding binary classifier is defined as $f_t(x;f)=\1_{\{f(x)>t\}}$.


For simplicity, we focus on the case where $G\in \{0, 1\}$ with $G=0$ corresponding to the non-protected class and $G=1$ corresponding to the protected. Extension to cases when the protected attribute is multi-labeled may be achieved using approaches in \cite{Jiang2020Jul,Miroshnikov2020}; see Appendix \ref{sec::multattr} for such extensions.

\subsection{Classifier fairness definitions and biases}

A common business use-case for models is in making binary classification decisions. For example, a credit card company may classify a prospective applicant as accepted or rejected. Because these decisions may have social consequences, it is important that they are fair with respect to sensitive demographic attributes. In this work, we focus on controlling deviations from parity-based (global) fairness metrics for ML models as described in  \cite{Jiang2020Jul,Vogel2021, Miroshnikov2020, Kwegyir-Aggrey2023, Becker2024}.  These global metrics are motivated by measures of fairness for classifiers \citep{Hardt2015,Feldman2015,Miroshnikov2020}, some of which we are given as follows.

\begin{definition}\label{def::genparity} Let $(X,G,Y)$ be a joint distribution as in Section \ref{sub:notation_and_hypotheses}. Suppose that $Y$ and $G$ are binary with values in $\{0,1\}$. Let $\hat{y}=\hat{y}(x)$ be a classifier associated with the response variable $Y$, and let $\Yhat=\hat{y}(X)$. Let $y^*\in\{0,1\}$ be the favorable outcome of $\Yhat$.
\begin{itemize}[label=$\bullet$]

\item  $\Yhat$ satisfies statistical parity if $\PP(\Yhat=y^*|G=0) = \PP(\Yhat=y^*|G=1).$

\item $\Yhat$ satisfies equalized odds if $\PP(\Yhat=y^*|Y=y,G=0) = \PP(\Yhat=y^*|Y=y,G=1)$, $ y\in\{0,1\}$

\item $\Yhat$ satisfies equal opportunity if $\PP(\Yhat=y^*|Y=y^*,G=0) = \PP(\Yhat=y^*|Y=y^*,G=1)$
\item Let $\mathcal{A}=\{A_j\}_{j=1}^M$ be a collection of disjoint subsets of $\Omega$. $\Yhat$ satisfies $\mathcal{A}$-based parity if
\begin{equation*}\label{genparity}
\PP(\Yhat=y^*|A_m, G=0)=\PP(\Yhat=y^*| A_{m}, G=1 ), \quad m \in \{1,\dots,M\}.
\end{equation*}
\end{itemize}
\end{definition}

Numerous works have investigated statistical parity \citep{Kamiran2009, Feldman2015, delBarrio, Jiang2020Jul} and equal opportunity \citep{Kamiran2010, Vogel2021} fairness criteria. Meanwhile, the $\mathcal{A}$-based parity may be viewed as a generalization of statistical parity, equalized odds, and equal opportunity biases. For example, letting $\mathcal{A}=\{\Omega\}$ produces the statistical parity criterion, letting $\mathcal{A}=\{\{Y=0\}, \{Y=1\}\}$ produces the equalized odds criterion, and letting $\mathcal{A}=\{\{Y=1\}\}$ produces the equal opportunity criterion. It may also be viewed as an extension of conditional statistical parity in \cite{verma2018} where the true response variable $Y$ is treated as a factor in determining fairness. 

The methods introduced in this work may be adapted to any $\mathcal{A}$-based parity criterion, but we focus on statistical parity (i.e., where $\mathcal{A}=\{\Omega\}$) for simplicity. We now present the definition of the classifier bias for statistical parity.

\begin{definition}\label{def::statbias} Let $\Yhat$, $y^*$, and $G$ be defined as in Definition \ref{def::genparity}. The classifier $\Yhat$ bias is defined as
\[\bias^C (\Yhat, G):=|\PP(\Yhat=y^*| G=0) - \PP(\Yhat=y^*| G=1 )|.
\]
\end{definition}

We may view the classifier bias as the difference in acceptance (or rejection) rates between demographics. Note that in some applications, we may prefer $\bias^{C}$ to be some other function of the rates $\PP(\Yhat=y^*| G=0)$ and $\PP(\Yhat=y^*| G=1 )$. For example, the ratio between these quantities is known as the adverse impact ratio (AIR) and may be written as
\begin{equation*}\label{eq:label}
{\rm AIR}(\Yhat | G) = \frac{\PP(\Yhat=y^*| G=1 )}{\PP(\Yhat=y^*| G=0)}\,.
\end{equation*}
In this case, fairness is achieved when ${\rm AIR}(\Yhat | G)=1$ so some natural classifier bias metrics based on AIR may be $1-{\rm AIR}$ (the negated AIR) or $-\log({\rm AIR})$ (the negated log AIR). Considering these alternatives may naturally lead one to consider a much broader family of bias metrics at both the classifier and model levels. 

To this end, we provide a generalization for the statistical parity bias using a cost function:
\begin{definition}\label{def::costclassifbias} Let $c(\cdot,\cdot) \geq 0$ be a cost function defined on $[0,1]^2$. Let $\Yhat$, $y^*$, and $G$ be defined as in Definition \ref{def::genparity}. The classifier $\Yhat$ bias associated with the cost function $c$ is defined by
\[\bias^C_c (\Yhat, G):=c(\PP(\Yhat=y^*| G=0),\PP(\Yhat=y^*| G=1)).\]
\end{definition}

\begin{remark} \rm
One can use $c(x,y)=d(x,y)^p$, where $d(\cdot,\cdot)$ is a metric on $\RR$, with $p \geq 1$.
\end{remark}

\subsection{Distribution-based fairness metrics}\label{sec::distfairness}
Businesses may seek to address model bias during the model development stage, even before detailed plans for the model's use have been made. To be specific, a single model $f=f(x)$ can be used to produce a range of classifiers $\{f_t\}_{t \in \RR}$ with different properties, and we may be unsure which classifiers will be selected for use in business decisions. To mitigate bias before this information is known, we require an appropriate definition of model bias. The work \cite{Miroshnikov2020} introduces model biases based on the Wasserstein metric as well as other integral probability metrics for fairness assessment of the model at the distributional level. Similar (transport-based) approaches for bias measurement have been discussed in \cite{Dwork2012,Jiang2020Jul, Kwegyir-Aggrey2023, Becker2024}.

\begin{definition}[\bf Wasserstein-1 model bias \cite{Miroshnikov2020}]\label{def::modbias} Let $(X,G)$ be as in Definition \ref{def::statbias}, and $f \in \mathcal{C}_{\mathcal{B}(\RR^n)}$ be a model with $\E[|f(X)|]<\infty$. The Wasserstein-1 model bias is given by 
\begin{equation}\label{modbias}
\Bias_{W_1}(f|X,G) = W_1(P_{f(X)|G=0},P_{f(X)|G=1}),
\end{equation}
where $P_{f(X)|G=k}$ is the pushforward probability measure of $f(X)|G=k$, $k \in \{0,1\}$, and $W_1(\cdot,\cdot)$ is the Wasserstein-1 metric on the space of probability measures $\mathscr{P}_1(\RR)$.
\end{definition}

It is worth noting that $\Bias_{W_1}(f|X,G)$ is the cost of optimally transporting the distribution of $f(X)|G=0$ into that of $f(X)|G=1$. This property leads to the bias explainability framework developed in \cite{Miroshnikov2020}.

In general, one can utilize $W_p$ metric, $p \geq 1$, for the bias measurement. However, the case $p=1$ is special, due to its relationship with statistical parity. It can be shown that the $W_1$-model bias is consistent with the statistical parity criterion as discussed in the lemma below, and which can be found in \cite{Jiang2020Jul,Miroshnikov2020}.

\begin{lemma}\label{lmm::ave_parity_transport}
Let a model $f$ and the random variables $(X,G)$ be as in Definition \ref{def::modbias}. Let $f_t(x) = \1_{\{f(x) > t\}}$ denote a derived classifier. The $W_1$-model bias can be expressed as follows:

\begin{equation}\label{modbiasgenparcons}
\begin{aligned}
\Bias_{W_1}(f|X,G) &= \int_0^1 |F^{[-1]}_{f(X)|G=0}(t)-F^{[-1]}_{f(X)|G=1}(t)| \, dt
= \int_{\RR} \bias^C (f_t | X, G)\, dt.
\end{aligned}
\end{equation}
\end{lemma}
\begin{proof}
The result follows from \citet{Shorack1986}.
\end{proof}

Thus, when $\Bias_{W_1}(f|X,G)$ is zero, there is no difference in acceptance rates between demographics for any classifier $\1_{\{f(x)>t\}}$ and (equivalently) no difference between the distributions $P_{f(X)|G=0}$ and $P_{f(X)|G=1}$. 

If $f$ is a classification score with values $f(x) \in [0,1]$, the relation \eqref{modbiasgenparcons} can be written as
\begin{equation}\label{scorebias}
\Bias_{W_1}(f|X,G) = \int_0^1 | F_{f(X)|G=0}(t)-F_{f(X)|G=1}(t)| dt = \E_{t\sim \mathcal{U}_{[0,1]}}[\bias^C (f_t | X, G)],
\end{equation}
where $\mathcal{U}_{[0,1]}$ is the uniform distribution on $[0,1]$ (e.g. see \cite{Jiang2020Jul}). This formulation lends the model bias a useful practical interpretation as the  average classifier bias across business decision policies $f_t$ where $t$ is sampled uniformly across the range $[0,1]$ of thresholds. When $f$ is a regressor with a finite  support, \eqref{scorebias} trivially generalizes to the integral  normalized by the size of the support \cite{Becker2024}.

A key geometric property of \eqref{scorebias} is that it changes in response to monotonic transformations of the model scores (in fact, it is positively homogeneous). In some cases, a distribution-invariant approach may be desired. To address this, \cite{Becker2024} has proposed a modification to \eqref{scorebias}  which removes its dependence on the model score distribution. Specifically, if the distribution of model scores $P_{Z}=P_{f(X)}$ is absolutely continuous with respect to the Lebesgue measure with the density $p_{Z}$, the distribution-invariant model bias for statistical parity is defined by
\begin{equation}\label{inv_bias}
{\rm bias}_{{\rm IND}}^{f}(f|X, G) := \int \bias^C (f_t | X, G)\cdot p_{Z}(t) \, dt =
\E_{t\sim P_{Z}}[\bias^C (f_t | X, G)].
\end{equation}

The distribution invariant model bias may be preferred over the Wasserstein model bias when one wants to measure bias in the rank-order induced by $Z$'s scores. Another method for measuring biases in a distribution invariant manner is to employ the ROC-based metrics of \cite{Vogel2021} which only depend on $Z$'s rank-order.

According to \cite{Becker2024}, when the score distributions are continuous, \eqref{inv_bias} equals $W_1(P_{F_{Z}(Z_0)},P_{F_{Z}(Z_1)})$, where $Z_k$, $k \in \{0,1\}$, are subpopulation scores. When $P_Z$ has atoms, the above relationship generally does not hold (see Example \ref{ex::disc_score_inv}). Nevertheless, it can be generalized. Specifically, we have the following result.

\begin{proposition}\label{prop:gen_invariant_bias}
Let a model $f$ and the distribution $(X,G)$ be as in Definition \ref{def::modbias}. Let $f_t(x) = \1_{\{f(x) > t\}}$, $Z=f(X)$, and $Z_k=f(X)|G=k$, $k \in \{0,1\}$, and let $\tF_Z$ be the left-continuous realization of $F_Z$. Then
\begin{equation}\label{inv_bias_gen}
{\rm bias}_{{\rm IND}}^{f}(f|X, G) := \int \bias^C (f_t | X, G) \, P_Z(dt)    = W_1(P_{\tF_{Z}(Z_0)},P_{\tF_{Z}(Z_1)}).
\end{equation}
\end{proposition}
\begin{proof}
The result follows from Lemma \ref{lmm::ave_parity_transport} and Corollary \ref{corr::abs-dist-cost}.
\end{proof}

In practice, even when the specific classifiers used in business decisions are unknown, knowledge about which thresholds (or quantiles) are more likely to be used typically exists. This is discussed in \cite{Vogel2021} where distribution invariant AUC-based metrics (used in bias mitigation) are restricted to an interval of interest  (typically determined by the business application) with the objective of improving the bias-fairness trade-off. See also Remark \ref{rem::nonunif-quantiles-bias}, which discusses a variation of \eqref{inv_bias} involving non-uniformly weighted quantiles.

For instance, there are applications where a threshold is chosen for business use according to some model $t=\tau(a)$, where $a\in \RR^m$  is a vector of observed external factors. In that case, given $a$, the statistical parity bias for the classifier $\hat{Y}_a(x):=f_{\tau(a)}(x)$ is given by
\[
\bias^C(\hat{Y}_a|X,G)=|F_{f(X)|G=0} (\tau(a)) - F_{f(X)|G=1}(\tau(a))|.
\]

When the value of $a$ is not known at the time of measurement, but is modeled as a realization of a random vector $A$, one can instead estimate the expected bias across all values $a \sim P_A$, which gives
\[
\E_{a \sim P_A}[ bias(\hat{Y}_a|X,G)] = \E_{t \sim \mu}[ bias(f_t|X,G)] = \int_0^1 | F_{f(X)|G=0}(t)-F_{f(X)|G=1}(t)| \, \mu(dt) 
\]
where $\mu=P_{\tau(A)}$ is the probability distribution of thresholds induced by $\tau(A)$.

This together with the above definitions of the bias motivates the following generalization.
\begin{definition}\label{def::cost_bias}
Let $c(\cdot,\cdot) \geq 0$ be a cost function on $\RR^2$, $f$ a model and $X,G,F_0,F_1$ as in Section \ref{sub:notation_and_hypotheses}. Let $\mu \in \mathscr{P}(\RR)$ be a Borel probability measure which encapsulates the importance of each threshold. Define
\begin{equation}\label{gen-cost-bias}
Bias^{(c)}_{\mu}(f|X,G):=\int c(F_0(t),F_1(t)) \, \mu(dt) = \E_{t \sim \mu} \big[c(F_0(t),F_1(t))\big].
\end{equation}
\end{definition}

The above formulation covers a large family of metrics that generalizes average statistical parity. Suppose $f$ is a classification score with values in $[0,1]$ and $\mu(dt) = \1_{[0,1]}dt$. Consider $c(x,y)=|x-y|^p$. When $p=1$, we obtain the average statistical parity which (in light of Lemma \ref{lmm::ave_parity_transport}) equals $W_1(Z_0,Z_1)$. For $p=2$, the metric  equals Cram\'{e}r's  distance \cite{cramer1928}, which coincides (in the univariate case) with the scaled energy distance \cite{Szekely1989}.  Finally, when $c(x,y)=|\log(x)-\log(y)|$, we get the absolute log-AIR.

In the spirit of \cite{Becker2024}, under certain conditions on {$\mu$}, one can express \eqref{def::cost_bias} as the minimal transportation cost with the cost function $c$ (see Definition \ref{def::transport-cost}). Specifically, we have the following result.

\begin{proposition}\label{prop::transp-cost-form}
Let $c(x,y)=h(x-y) \geq 0$, with $h$ convex. Let $X,G,Z_k,F_k$ and $\mu$ be as in Definition \ref{def::cost_bias}. Suppose the supports of $P_{Z_0}$, $P_{Z_1}$ and $\mu$ are identical and connected. Finally, suppose  the CDFs $F_{0},F_{1}, F_{\mu}$ are continuous and strictly increasing on their supports. Then
\[
Bias^{(c)}_{\mu}(f|X,G) = \int c(F_0(t),F_1(t)) \, \mu(dt) = \mathscr{T}_c({F_0}_{\#}\mu,{F_1}_{\#}\mu).
\]

where $\mathscr{T}_c$ is the minimal transport cost from ${F_0}_{\#}\mu$ to ${F_1}_{\#}\mu$ for the cost $c$.
\end{proposition}
\begin{proof}
See Appendix \ref{app::prop::transp-cost-form}.
\end{proof}

\begin{remark} \rm
Proposition \ref{prop::transp-cost-form} remains valid if $c(x,y)=d(x,y)^p$, where $d(\cdot,\cdot)$ is a metric, with $p \geq 1$. In that case, we have 
\[\mathscr{T}_c({F_0}_{\#}\mu,{F_1}_{\#}\mu)^{1/p}=W_p({F_0}_{\#}\mu,{F_1}_{\#}\mu;d).\]
\end{remark}

\subsection{Model explainability}\label{sec::modelexplainability}

Due to regulations, model explainability is often a crucial aspect of using models to make consequential decisions. Therefore, this work seeks to mitigate bias in models while preserving explainability. Following \cite{miroshnikov2024}, we define a generic model explanation method.

\begin{definition}\label{def::genexplainer}
Let $X=(X_1, ..., X_n)$ be predictors. A local model explainer is the map $x \to E(x;f,X)\allowbreak=(E_1,\dots,E_n)$ that quantifies the contribution of each predictor $X_i$, $i\in N:=\{1,\dots,n\}$, to the value of a model $f \in \mathcal{C}_{\mathcal{B}(\RR^n)}$ at a data instance $x \sim P_X$.  The explainer is called additive if $f(x)=\sum_{i=1}^n E_i(x;f,X)$. 
\end{definition}

The additivity notion can be slightly adjusted to take into account the model's expectation.

\begin{definition}\label{def::genexplainer-centered}
The explainer $E(\cdot;f,X)$ is called $P_X$-centered if $\E_{x \sim P_X}[E_i(x;f,X)]=0$, $i \in N$. We say that $E$ satisfies $P_X$-centered additivity if $f(x)-\E[f(X)]=\sum_{i=1}^n E_i(x;f,X)$.
\end{definition}

In practice, model explanations are meant to distill the primary drivers of how a model arrives at a particular decision, and the meaningfulness of the model explanation depends on the particular methodology.

Some explanation methodologies of note include global methods such as \cite{Friedman,Lakkaraju2017} which quantify the overall effect of features, local methods such as locally-interpretable methods \cite{Ribeiro2016, hu2018}, and methods such as \cite{Strumbelj2014,LundbergLee,ChenLShapley2019} which provide individualized feature attributions based on the  Shapley value \cite{Shapley}.

The Shapley value, defined by 
\begin{equation*}\label{shap-form}
\varphi_i[N, v] := \sum_{S\subseteq N\backslash \{i\}} \frac{|S|!(|N|-|S|-1)!}{|N|!}(v(S\cup \{i\})-v(S)) \quad (i\in N=\{1,2,\dots,n\}),
\end{equation*}
where $v$ is a cooperative game (set function on $N$) with $n$ players,  is often a popular choice for the game value (in light of its properties such as symmetry, efficiency, and linearity), but other game values and coalitional values (such as the Owen value \cite{Owen}) have also been investigated in the ML setting \cite{Wang-Lundberg,filom2024,Kotsiopoulos2023,miroshnikov2024}.

In the ML setting, the features $X=(X_1,X_2,\dots,X_n)$  are viewed as $n$ players in a game $v(S;x,X,f)$, $S \subseteq N$,  associated with the observation $x\sim P_X$, random features $X$, and model $f$. The game value $\varphi_i[N,v]$ then assigns the contributions of each respective feature to the total payoff $v(N;x,X,f)$ of the game. Two of the most notable games in the ML literature \cite{Strumbelj2014,LundbergLee} are given by
\begin{equation}\label{margcondgamedet}
\vce(S; x,X,f)=\E[f(X)|X_S=x_S], \quad \vme(S;x,X,f)=\E[f(x_S,X_{-S})],
\end{equation}
 where $\vce(\varnothing; x, X, f)=\vpdp(\varnothing; x, X, f):=\E[f(X)]$.

The efficiency property of $\varphi$ allows the total payoff $v(N)$ to be disaggregated into $n$ parts that represent each player's contribution to the game: $\sum_{i=1}^n \varphi_{i}[N,v] = v(N)$. The games defined in \eqref{margcondgamedet} are not cooperative, as they do not satisfy $v(\varnothing)=0$. In this case, the efficiency property takes the form:
\[
\sum_{i=1}^n \varphi_{i}[N,v] = v(N)-v(\varnothing) = f(x)-\E[f(X)], \quad v \in \{\vce(\cdot;x,X,f),\vme(\cdot;x,X,f)\}.
\]

An important property of the games in \eqref{margcondgamedet} is that of linearity with respect to models. Since $\varphi[N,v]$ is linear in $v$, the linearity (with respect to models) also extends to the marginal and conditional Shapley values. That is, given two continuous bounded models $f,g$ we have
\begin{equation*}\label{shap-linearity}
\varphi[N,v(\cdot;X,\alpha \cdot f+g)]=\alpha \cdot \varphi[N,v(\cdot;X,f)]+\varphi[N,v(\cdot;X,g)], \quad  v\in\{\vce,\vme\}.
\end{equation*}

For simplicity, this work explores explainability through the feasibility of computing marginal Shapley values, with $\varphi_i^{\ME}(x,f):=\varphi_i[N, v^{\ME}(\cdot; x, X, f)]$, $i \in N$, denoting the marginal Shapley value of the $i$-th predictor. Computing such explanations is an interesting challenge owing to their prominence in the ML literature and exponential $O(2^n)$ computational complexity. However, our proposed bias mitigation approaches are compatible with any explainability method satisfying linearity.

\section{Bias metrics approximations for stochastic gradient descent}\label{sec::biasapprox}

In this subsection, we consider approximations of the bias metrics discussed in Section \ref{sec::distfairness} that allow us to employ gradient-based optimization methods. Here, for simplicity, we focus on classification score models, whose subpopulation distributions are not necessarily continuous. 

In what follows, we let $f=f(\cdot;\theta)$ denote a predictive model, parameterized by $\theta$, with values in $\RR$, and $F_k(\cdot;\theta)$ be the CDF of $Z_k ^{\theta}\sim P_{f(X;\theta)|G=k}$, $k\in\{0,1\}$.  To simplify the exposition, we assume that the cost function $c(a,b)=h(a-b)$, where $h \geq 0$ is Lipschitz continuous on $[-1,1]$, and let $\mu \in \mathscr{P}(\RR)$ denote a Borel probability measure, describing the distribution of thresholds.  Consider the bias metric
\begin{equation}\label{gen_bias}
Bias^{(h)}_{\mu}(f(\cdot;\theta)|X,G):= \int h(F_0(t;\theta)-F_1(t;\theta)) \, \mu(dt) \, < \, \infty.
\end{equation}

Clearly, \eqref{gen_bias} depends on the parameter $\theta$ in an intricate way and care must be taken to differentiate this quantity or its approximation with respect to $\theta$. For motivation, note that when one only has access to a finite number of samples $x_k^{(1)}, \dots, x_k^{(m)} \sim P_{X | G=k}$, we may seek to substitute the CDFs $F_k$ with their empirical CDF analogs when computing metrics. In this case, we have
\[
\hat{F}_k(t;\theta) = \sum_{i=1}^m \1_{\{f(x_k^{(i)};\theta) \leq t \}}\,.
\]

However, due to the presence of indicator functions, $\hat F_k(t{\color{brown};}\theta)$ is in general not differentiable in $\theta$. Thus, substituting $F_k$ for $\hat F_k$ in \eqref{gen_bias} may result in bias metrics that are not differentiable. To address this issue, we adopt a relaxation of the formulation \eqref{gen_bias}, inspired by \cite{Vogel2021}, which enables the construction of various differentiable approximations suited to stochastic gradient descent.

\subsection{Relaxation approximation}\label{sec::relax_approx}

Let $H(z)=\1_{\{z>0\}}$ be the left-continuous version of the Heaviside function. Let $\{r_s(t)\}_{s \in \RR_+}$ be a family of continuous functions such that $r_s(z)$ is non-decreasing and Lipschitz continuous on $\RR$, and satisfies $r_s(z) \to 0$ as $z \to -\infty$,  $r_s(z)\to 1$ as $z \to \infty$, and $\lim_{s \to \infty} r_s(z)=H(z)$ for all $z \in \RR$. Now, define functions

\begin{equation}\label{relax_cdf}
F^{(s)}_k(t;\theta):=1-\E[r_s(Z_k^{\theta}-t)], \quad k \in \{0,1\}.
\end{equation}
By Lemma \ref{lmm::relax_stat}, $F^{(s)}_k$ is a globally Lipschitz CDF approximating $F_k$, where $\lim_{s \to \infty} F_k^{(s)}(t;\theta)=F_k(t;\theta)$, $\forall t \in \RR$, and 
\begin{equation}\label{relax_limit}\tag{RL}
Bias^{(h)}_{\mu}(f|X,G):= \int h(F_0(t;\theta)-F_1(t;\theta)) \mu(dt) = \lim_{s \to \infty} \int  h \big( F_0^{(s)}(t;\theta) - F_1^{(s)}(t;\theta) \big) \mu(dt).
\end{equation}

Clearly, \eqref{relax_limit} suggests that one can approximate \eqref{gen_bias} by computing the bias between smoother CDFs $F_{k}^{(s)}$. Furthermore, it can be shown that their estimators are also differentiable w.r.t. $\theta$. To this end, define
\begin{equation*}\label{loc_bias_relax}
B(t;\theta) := F_0(t;\theta) - F_1(t;\theta), \quad B_s(t;\theta) := F^{(s)}_0(t;\theta) - F^{(s)}_1(t;\theta).
\end{equation*}
Let $D_k=\{x^{(1)}_k,\dots,x^{(m_k)}_k\}$ be a dataset of samples from the distribution $P_{X|G=k}$, $k \in \{0,1\}$. Then the estimator of $B_s(t;\theta)$ can be defined in terms of subpopulation model scores:
\begin{equation}\label{Bs_estimator}
\hat{B}_{s,m_0,m_1}(t; \theta) := \frac{1}{m_1} \sum_{i=1}^{m_1} r_s(f(x^{(i)}_1;\theta)-t) - \frac{1}{m_0} \sum_{i=1}^{m_0} r_s(f(x^{(i)}_0;\theta)-t).
\end{equation}

Note that if $r_s$ is differentiable, then the map $(t,\theta) \mapsto \hat{B}_{s,m_0,m_1}(t; \theta)$ is differentiable -- assuming, of course, that the map $\theta \mapsto f(\cdot; \theta)$ is differentiable. If $r_s$ is globally Lipschitz, the weak gradient $\nabla_{\theta}\hat{B}_{s,m_0,m_1}$ is well-defined.

\begin{remark} \rm
Here are two examples of the relaxation family $\{r_s\}_{s \in \RR_+}$. Define $r_s(z)=r(s z)$, where $r(z)=0$ for $z \leq 0$, $r(z)=z$, for $z \in (0,1)$ and $r(z)=1$, for $z \geq 1$. Alternatively, set $r_s(z)=\sigma(s  (z-\frac{1}{\sqrt{s}}))$ where $\sigma(z)$ is the logistic function. In this case $F_k^{(s)}$ are infinitely differentiable.
\end{remark}

\begin{remark} \rm
If $\mu$ is atomless, the requirement $\lim_{s \to \infty} r_s(0)=0=H(0)$ can be dropped, in which case  \eqref{relax_cdf} still holds, and $\lim_{s \to \infty} F^{(s)}_k(t) \to F_k(t)$ at any points of continuity of $F_k$; see Lemma \ref{lmm::relax_stat}. For instance, one can use  $r_s(z)=\sigma(sz)$ where $\lim_{s \to \infty} r_s(0)=\frac{1}{2}$. 
\end{remark}

\subsection{Bias estimators}

Following the discussion above, we propose several methods to estimate the relaxed bias metric
\begin{equation}\label{relax_bias}
Bias^{(h)}_{\mu,s}(f(\cdot;\theta)|X,G):= \int h(F^{(s)}_0(t;\theta)-F^{(s)}_1(t;\theta)) \mu(dt) = \int h(B_s(t;\theta)) \mu(dt)
\end{equation}
employing the estimator \eqref{Bs_estimator}. Throughout, we assume that $\Lip (r_s) \leq s$ for all $s \in \RR_+$, and that $\mu$ is fixed. We extend the estimators introduced below to cases where $\mu$ depends on $(f, P_G, P_{X|G})$, such as the metrics in \cite{Vogel2021, Becker2024}, in Appendix \ref{sec::invar_metrics}. 


\subsubsection{Threshold-MC estimator} 

Let $D_{\tau}=\{t^{(1)},\dots,t^{(T)}\}$ be samples from the distribution $\mu(dt)$.  Then
\begin{equation}\label{full-mc-est}
\begin{aligned}
\int h \big( B_s(t;\theta) \big) \mu(dt) = \E_{t \sim \mu} [ h(B_s(t))] \approx \frac{1}{T} \sum_{j=1}^T h(B_s(t^{(j)};\theta))  \approx \frac{1}{T} \sum_{j=1}^T h(\hat{B}_s(t^{(j)};\theta)).
\end{aligned}
\end{equation}

More formally, we have the following approximation result.
\begin{theorem}\label{thm::MCestimator}
Let $h \geq 0$ be Lipschitz on $[-1,1]$, and let $\mu \in \mathscr{P}(\RR)$. For $k \in \{0,1\}$, let  ${\bf D}_k = \{{\bf x}^{(1)}_k,\dots,{\bf x}^{(m_k)}_k\}$ be i.i.d.\ samples  from $P_{X | G=k}$, and let  ${\bf D}_{\tau} = \{{\bf t}^{(1)},\dots,{\bf t}^{(T)}\}$ be i.i.d. samples from $\mu$, independent of ${\bf D_0} \cup {\bf D}_1$. Let $B_s(t;\theta)$ be defined as in \eqref{loc_bias_relax}, and let $\hat{\bf B}_{s,m_0,m_1}(t;\theta)$ denote its estimator from \eqref{Bs_estimator}, evaluated using the samples ${\bf D}_k$. Then,
\begin{equation}\label{full-mc-est-exact}
\frac{1}{T} \sum_{j=1}^T h\big(\hat{\bf B}_{s,m_0,m_1}({\bf t}^{(j)};\theta)\big) = \int h\big(B_s(t;\theta)\big) \, \mu(dt) + \mathcal{E}^{(s)}_{m_0,m_1,T},
\end{equation}
where $\mathcal{E}^{(s)}_{m_0,m_1,T} \to 0$ in $L^2(\P)$ as $m_0, m_1, T \to \infty$, with an optimal rate of $O(T^{-1/2})$,  uniformly in $s \in \RR^+$,  when $T = c \cdot \min(m_0, m_1)$ for some fixed constant $c$.
\end{theorem}

\begin{proof}
See Appendix \ref{proof::MCestimator}.
\end{proof}

As a consequence of the above theorem, the left-hand side of \eqref{full-mc-est} is a (weakly) consistent estimator of the integral on the right. Also, note that $h$ is Lipschitz on $[-1,1]$ containing the image of both $B_s$ and $\hat{B}_s$.

\begin{remark}\rm
For the estimator \eqref{full-mc-est},  $\mu$ is not required to have a density (that is, it can contain atoms) and can be easily adapted to invariant metrics, such as $\mu = P_{f(X;\theta)}$ from \cite{Becker2024}; see Appendix \ref{sec::invar_metrics}.
\end{remark}

\subsubsection{Threshold-discrete estimator}

In this section, for simplicity of exposition, we assume that $f=f(\cdot;\theta)$ is a classification score function taking values in $[0,1]$, and that the measure $\mu(dt)=\rho(t)dt$, where $\rho$ is Lipschitz continuous function supported on $[0,1]$. Given a uniform partition $\mathcal{P}_T := \{t_0=0 < t_1 < \dots < t_T=1\}$  of $[0,1]$, we consider a threshold-discretized bias estimator:
\begin{equation}\label{int_approx_est}
\int_0^1 h \big( B_s(t;\theta) \big) \mu(dt) = \int_0^1 h \big( B_s(t;\theta) \big) \rho(t) dt \approx \Big( \sum_{j=1}^T h(\hat{B}_s(t_j;\theta)) \rho(t_j) \Delta t \Big)
\end{equation}
where the approximation error scales as $O((1+s)/T)$ in $L^2(\P)$ as established by the theorem below.
 
\begin{theorem}\label{thm::DiscrEstimator}
Let $h$, $\mu$, ${\bf D}_0$, ${\bf D}_1$, $B_s$, and $\hat{\bf B}_{s,m_0,m_1}$ be as in Theorem \ref{thm::MCestimator}. 
 Assume $\mu(dt)=\rho(t)dt$ on $[0,1]$, where $\rho(t)$ is Lipschitz continuous on $[0,1]$, and $\operatorname{supp}(\mu) \subseteq [0,1]$. Suppose the relaxation function $r_s$, from Definition \ref{relax_family}, is Lipschitz on $\RR$ with $\Lip(r_s)\le sC_r$ for some constant $C_r$. Let $\mathcal{P}_T = \{t_0 =0 < \dots < t_T=1\}$ be the uniform partition of $[0,1]$, with $\Delta t = t_{i+1}-t_i=\frac{1}{T}$. Then, for any $s>0$, 
\begin{equation}\label{int_approx_est_exact}
\frac{1}{T}\sum_{j=1}^T h(\hat{\bf B}_{s,m_0,m_1}(t_j;\theta)) \rho(t_j) = \int_0^1 h \big( B_s(t;\theta) \big) \rho(t) dt + \mathcal{E}^{(s)}_{m_0,m_1,T},
\end{equation}
where $\mathcal{E}^{(s)}_{m_0,m_1,T}\to 0$ in $L^2(\P)$ as $m_0,m_1,T\to \infty$ with an optimal rate of $O((1+s)/T)$ when $T/(1+s)=c\cdot \sqrt{\min(m_0,m_1)}$ for some fixed constant $c$.
\end{theorem}

\begin{proof}
See Appendix \ref{proof::DiscrEstimator}.
\end{proof}

Note that if $r_s$ is differentiable, then the estimators in \eqref{full-mc-est} and \eqref{int_approx_est} are differentiable with respect to $\theta$ in view of \eqref{Bs_estimator}. Similar conclusions to those in Section \ref{sec::relax_approx} apply if $r_s$ is Lipschitz continuous on $\RR$. 

Moreover, if $s=O(T^{\alpha})$ for some $\alpha \in (0,1)$, then the error in Theorem \ref{thm::DiscrEstimator} has rate $O(T^{\alpha-1})$ as $T \to \infty$. In particular, if $\alpha>0.5$, this rate is faster than the  $O(T^{-1/2})$ rate in Theorem \ref{thm::MCestimator}.

Finally, note that the estimator \eqref{int_approx_est} and the proof of Theorem \ref{thm::DiscrEstimator} can be readily adapted to cases where the density $\rho$ has finite support or decays exponentially fast.

\begin{remark}\label{rem::higherorder}\rm
The approximation in \eqref{int_approx_est} may be improved by using higher order numerical integration schemes. For example, if $h$ and $r_s$ are twice continuously differentiable with bounded first and second derivatives on $[-1,1]$ and $\RR$, respectively, then using the trapezoidal rule, we obtain  the error $O( ((s+1)\Delta t)^2)$, where we assumed $|r_s''| \leq s^2$.
\end{remark}

\begin{remark}\rm
To ensure numerical convergence of approximation \eqref{int_approx_est} to \eqref{gen_bias} as $\Delta t\to0$ and $s\to\infty$, we see from Theorem \ref{thm::DiscrEstimator} that $s$ must tend to infinity in such a way that $s \Delta t \to 0$.
\end{remark}

\begin{remark}\rm
If $h(\cdot)=|\cdot|$ and $\mu(dt)=dt$, one can drop the relaxation, and instead consider the estimator from \cite{Jiang2020Jul}
\begin{equation}\label{norelax-est}
\int_0^1 |B(t;\theta)| dt = W_1(P_{f(X;\theta)|G=0},P_{f(X;\theta)|G=1}) \approx W_1(\hat{P}_{f(X;\theta)|G=0},\hat{P}_{f(X;\theta)|G=1}),
\end{equation}
where $\hat{P}_{f(X;\theta)|G=k}$, $k \in \{0,1\}$, denotes the empirical distribution of the model's scores. The gradient of the right-hand side in \eqref{norelax-est} can  then be computed explicitly using the optimal transport coupling; see \cite[Lemma 3]{Jiang2020Jul}.
\end{remark}

\subsubsection{Energy estimator} 

Let us assume that $h(\cdot)=2|\cdot|^2$ in \eqref{gen_bias}, and that $\mu$ is atomless. Then, by Proposition \ref{prop::quant_transf_bias}, the bias metric can be expressed as the twice Cram\'{e}r's distance \cite{cramer1928}: 
\begin{equation}\label{energy_bias_mu}
\begin{aligned}
& Bias^{(h)}_{\mu}(f|X,G)) \\
& \quad = 2\int (F_0(t)-F_1(t))^2 \, \mu(dt) 
   = 2\int_0^1 (F_{S^{(\mu)}_0}(q)-F_{S_1^{(\mu)}}(q))^2 \, dq\\
 & \quad = 2 \int_0^1 |s_0-s_1| [P_{S^{(\mu)}_0} \otimes P_{S^{(\mu)}_1}](ds_0,ds_1)  
   - \sum_{k \in \{0,1\}} \int |s_k-\tilde{s}_k| \, [P_{S^{(\mu)}_k} \otimes P_{S^{(\mu)}_k}](ds_k,d\tilde{s}_k),
\end{aligned}
\end{equation}
where $S_{k}^{(\mu)}=F_{\mu}(Z_k)$, and where in the last equality we used the fact that the twice Cram\'{e}r's distance coincides with the squared energy distance \cite{Szekely1989}.

Let $z^{(i)}_k=f(x^{(i)}_k;\theta)$, where $x_k^{(i)} \in D_k$, $k \in \{0,1\}$. Then, since $F_{\mu}(z^{(i)}_k) \sim P_{S_k^{(\mu)}}$, the $E$-statistic \cite{Szekely1989} 
\begin{equation}\label{E-statistic-mu}
\cE^{(\mu)}_{m_0,m_1} := \frac{2}{m_0 m_1} \sum_{i=1}^{m_0}\sum_{j=1}^{m_1} |F_{\mu}(z_0^{(i)})-F_{\mu}(z_1^{(j)})| -  \sum_{k \in \{0,1\}}\frac{1}{m_k^2} \sum_{i=1}^{m_k}\sum_{j=1}^{m_k} |F_{\mu}(z_k^{(i)})-F_{\mu}(z_k^{(j)})|,
\end{equation}
which is always non-negative, can be used to estimate \eqref{energy_bias_mu} with estimation error $O(1/\sqrt{m_0\cdot m_1})$; for details see \cite{Szekely1989}.

\begin{remark} \rm
We note that if $F_{\mu}$ is differentiable, then \eqref{E-statistic-mu} is differentiable with respect to $\theta$. Similar conclusions to those in Section \ref{sec::relax_approx} apply to \eqref{E-statistic-mu} if $F_{\mu}$ is  Lipschitz continuous on $\RR$. 
\end{remark}

\begin{remark}\rm
If $\mu$ contains atoms, then the estimator \eqref{E-statistic-mu} can be adjusted by replacing $F_{\mu}$ with $F_{\mu}^{(s)}$, which denotes its relaxed version. This results in an approximation of the relaxed bias metric, that is, \eqref{energy_bias_mu} with $\mu$ replaced by $\mu^{(s)}$, a probability measure with CDF $F_{\mu}^{(s)}$, allowing for a differentiable estimator even when $\mu$ is not atomless.
\end{remark}

\begin{remark} \rm
The relaxation limit \eqref{relax_limit} and the estimators in \eqref{full-mc-est}, \eqref{int_approx_est}, and \eqref{E-statistic-mu},  can be generalized to any cost function $c(\cdot,\cdot)$ which is continuous on $[0,1]^2$.
\end{remark}

\section{Bias mitigation via model perturbation}
\label{sec::methods_perturb}

In this section, we introduce novel post-processing methods for explainable bias mitigation without access to demographic information at inference time. By ``explainable'', we refer to the property that the existing explainability algorithms -- such as those used to compute the marginal game value\footnote{Explanations based on game values are often designed as post-hoc techniques (for example, \cite{LundbergTreeSHAP,LundbergIntervTreeSHAP}), but they may naturally arise in some cases as explanations of inherently interpretable models \cite{filom2024}.} (e.g., Shapley or Owen values) or other types of explanations -- can be directly and efficiently applied to the post-processed models. This avoids the need to design new explanation techniques,  or rely on existing general purpose techniques which are often computationally expensive; see Section \ref{sec::outputperturbation} for further details.

\subsection{Demographically blind optimization with global fairness constraints}

To motivate our approaches, consider a general setting for demographically-blind fairness optimization. Let  $\cF$ be a parametrized collection of models,
\begin{equation*}\label{model_family}
\cF :=\big\{f(x; \theta) \in \mathcal{C}_{\mathcal{B}(\RR^n)}, \,\,\, \theta \in \Theta \big\},
\end{equation*}
where $\Theta$ denotes a parameter space, $(X,Y,G)$ a joint distribution as in Section \ref{sub:notation_and_hypotheses}, $L(y,f(x))$ a loss function, and $\Bias(f|X,G)$ a non-negative bias functional. Define:
\begin{equation*}\label{param_loss_bias}
\cL(\theta) := \E[L(f(X;\theta),Y)], \quad \B(\theta):=\Bias(f(\cdot,\theta)| X,G), \quad \theta \in \Theta.
\end{equation*}

In the fairness setting, one is interested in identifying models in $\cF$ whose bias-performance trade-off is optimal, -- that is, among models with similar performance, one would like to identify those that are the least biased. More precisely, for each $b \geq 0$, define
\[
\Theta_b := \{\theta \in \Theta: \B(\theta) \leq b\}.
\] 
Then, given $b \geq 0$, minimize $\cL$ on $\Theta_b$, that is, find $\theta_{b}^*$ for which $\B(\theta_{b}^*) \leq b$ and $\cL(\theta_{b}^*) \leq \cL(\theta)$, $\theta \in \Theta_b$. Varying the parameter $b$ in this constrained minimization defines the bias-performance efficient frontier. 

Thus, constructing the efficient frontier of the family $\family$ amounts to solving a constrained minimization problem, which can be reformulated in terms of generalized Lagrange multipliers using the Karush-Kuhn-Tucker approach
\citep{Karush,KuhnTucker}:
\begin{equation}\label{minfront}\tag{BM}
\begin{aligned}
\theta^*(\omega) &:= \underset{\theta \in \Theta}{\rm argmin} \big\{  \cL(\theta) + \omega \B(\theta)\big\}, \quad \omega\geq 0.
\end{aligned}
\end{equation}
Here, $\omega$ denotes a bias penalization coefficient that implicitly corresponds to a constraint level $b$; varying $\omega$ traces out the efficient frontier.

The choice of $\cF$ in \eqref{minfront} matters as it leads to conceptually distinct bias mitigation approaches:

\begin{itemize}
  \item [(A1)] Optimization performed during the ML training. In this case, $\cF$ is a family of machine learning models  (e.g. neural networks, tree-based models, etc), and $\Theta$ is the space of model parameters.

  \item [(A2)] Optimization via hyperparameter selection. Here, $f(x;\theta)$ denotes the trained model given a hyperparameter $\theta$, where we suppress the dependence on the fitted model parameters (e.g., weights) for simplicity. The construction proceeds in two steps. First, for a given $\theta$,  training is performed without fairness constraints. Then $\theta$ is adjusted to minimize \eqref{minfront}. Note that adjusting $\theta$ entails retraining.

  \item [(A3)] Optimization is performed over a family   of post-processed models,  performed after training. Namely, given a trained model $f_*$,  the family is constructed by adjusting $f_*$ with $\Theta$ denoting a space of adjustment parameters.

\end{itemize}

 The problem \eqref{minfront} is not trivial for the following reasons. First, the optimization is in general non-convex, which is a direct consequence of the loss and bias terms in the objective function. Second, the dimensionality of the parameter $\theta$ can be high, increasing the complexity of the problem. Finally, in applications where the map $\theta \to f(\cdot;{\theta}) \in \family$ is non-smooth (e.g., discontinuous),  utilizing gradient-based optimization techniques might not always be feasible 
 or straightforward. For example, tree-based models like GBDTs require adapting gradient boosting techniques to handle problem with fairness constraints such as \eqref{minfront}; see \cite{pangia2024}.

There are numerous works proposed in the literature that address \eqref{minfront} in the settings of (A1)-(A3). For approach (A1) where the fairness constraint is incorporated  directly into training, see \cite{Feldman2015,Dwork2012, Zemel2013, Woodworth2017} for its applicability to classifier constraints and \cite{Jiang2020Jul,Vogel2021, pangia2024} for its applicability to global constraints.

For approach (A2) which performs hyperparameter search (using random search, Bayesian search or feature engineering), see \cite{Perrone2020, schmidtpatent} for an application of hyperparameter tuning to bias mitigation and \cite{Bergstra2011} for generic hyperparameter tuning methodologies.

For approach (A3), see the paper \cite{Miroshnikov2020b} and the patent publication \cite{miroshnikovpatent} where the family of post-processed models is constructed from functions of the form $f(\cdot;\theta)=f_*\circ T_{\theta}(x)$, where $T_{\theta}$ is a parametrized transformation. The minimization is then done using derivative-free methods such as Bayesian search to accommodate various metrics and allow for the trained model $f_*$ to be discontinuous. To reduce the dimensionality of the problem, the parametrized transformations are designed using the bias explanation framework of \citep{Miroshnikov2020}.

The post-processing methodologies in \cite{Chzhen2020, Jiang2020Jul, Kwegyir-Aggrey2023,miroshnikovpatent} that make use of optimal transport techniques also fall under purview of (A3), though the methods in \cite{Chzhen2020, Jiang2020Jul, Kwegyir-Aggrey2023} are not demographically-blind. In these works, the distribution $P_{f_*(X)}$ is assumed to have a density. Then the family $\cF$ is obtained by considering linear combinations of the trained model $f_*$ and the repaired model $\bar{f}(X,G)$, which is constructed using Gangbo-\'{S}wi\c{e}ch maps between subpopulation distributions $P_{f(X)|G=k}$, $k \in \cG$, and their $W_1$-barycenter. For the method to be demographically blind, the explicit dependence on $G$ could be removed by projecting the repaired model as proposed in \cite{miroshnikovpatent}; also see Section \ref{app::opttransp}. In this case, $\cF$ is a one-parameter family of models and optimization is not needed to find its efficient frontier.

In what follows, motivated by the optimization ideas of \cite{Jiang2020Jul, Vogel2021}, we propose new scalable bias mitigation approaches that solve \eqref{minfront} over families of explainable post-processed models without explicit dependence on $G$. In particular, model score outputs are adjusted (e.g., by perturbing the model components) rather than their inputs as in \cite{Miroshnikov2020b}, allowing gradient descent to be used instead of Bayesian optimization even if using GBDTs. This allows us to optimize over larger model families which may yield better efficient frontiers.

\subsection{Explainable bias mitigation through output perturbation}\label{sec::outputperturbation}

We now outline the main idea for how to construct a family of perturbed explainable models. Suppose $f_*$ is a trained regressor model. Let  $w=(1,w_1,\dots, w_m)(x;f_*)$ be weight functions (or encoders), whose selection is discussed later. The family of models about $f_*$ associated with the weight map $w$ is then defined by
\begin{equation}\label{linearfam}
\cF(f_*;w) := \bigg\{f: f(x;\theta):=f_*(x) - \theta \cdot w(x;f_*), \,\, \theta=(\theta_0,\dots,\theta_m) \in \Theta \subseteq \RR^{m+1} \bigg\},
\end{equation}
where $\theta \in \Theta$ is a learnable parameter\footnote{For calibration purposes, one can introduce an additional parameter $\alpha$ in \eqref{linearfam}, yielding models of the form $\alpha f_* -\theta \cdot w(x;f_*)$.}. 


\begin{remark} \rm
In some applications, the map $w$ may depend on the distribution of $X$ as well as the model representation $\cR(f)$ in terms of basic ML model structures. In such cases, we write $w=w(\cdot;f,X,\cR(f))$.
\end{remark}

If the trained model is used in classification, the above family is slightly adjusted. Specifically, let $g_*=\sigma \circ f_* $ denote the trained model outputting probability estimates, where $\sigma$ is a link function (e.g., logistic) and $f_*$ is a raw score. We consider  the minimization problem \eqref{minfront} over the family $\cF(f_*;w)$ for the raw score $f_*$ (rather than $g_*$) with adjusted loss and bias metrics: 
\begin{equation}\label{clf-loss-bias}
\cL(\theta):=\E[L(\sigma \circ f_*,Y)], \quad \B(\theta):=Bias(\sigma \circ f_*|X,G).
\end{equation}

It is crucial to point out that the minimization problem \eqref{minfront} over the family \eqref{linearfam} is linear in $w$, since the map $\theta \mapsto \theta \cdot w$ is linear. As we will see, this setup (unlike in \cite{Miroshnikov2020b}) circumvents the lack of differentiability of the trained model and allows for the use of gradient-based methods, even when $f_*$ is discontinuous.

Furthermore, given an explainer map $(x,f,X) \mapsto E(x;f,X) \in \RR^n$, assumed to be linear in $f$ and centered (i.e. $E(x;{\rm const},X)=0$), the explanations of any element of \eqref{linearfam} can be expressed in terms of explanations of the trained model and those of the weight functions:
\begin{equation}\label{familyexpl}
E(x;f(\cdot;\theta),X)=E(x;f_*,X)-\sum_{j=1}^m\theta_j E(x;w_j,X), \quad f(\cdot;\theta) \in \cF(f_*;w).
\end{equation}

 The family \eqref{linearfam} and the explanation relation \eqref{familyexpl} are central to our framework. The representation of postprocessed models enables interpretation of bias adjustments, which -- expressed in terms of encoders -- are explicitly given by $\theta_j w_j$. At the same time, the linear relation \eqref{familyexpl} is particularly useful in industrial applications, where feature-level explanations for a set of models along the bias-performance efficient frontier must be computed quickly across a large dataset of individuals. 

 In our setting, we assume the existence of an explainability algorithm that efficiently computes explanations for the trained model $f_*$ and the encoders $w_j$. By precomputing explanations of the trained model and weights for each individual, the corresponding explanations of any post-processed model in the family can be determined via linear combination. Thus, in view of \eqref{familyexpl},  the total cost of generating new explanations is determined entirely by the complexity of the explanation algorithm for the trained model and that of the encoders.

For instance, when $E=\varphi^{\ME}$, relation \eqref{familyexpl} holds. In this case, it is essential to use explanation algorithms that efficiently compute the marginal Shapley value for both the trained model and the encoders. 

For example, if $f_*$ is a tree ensemble and $T$, $L$ and $D$ denote its number of trees, maximum number of leaves, and background dataset respectively, the per-instance complexity of explanations is $O(T \cdot L \cdot |D|)$ for interventional TreeSHAP \cite{LundbergIntervTreeSHAP} and $O(T\cdot \log{L})$ for the algorithm of \cite{filom2024} for oblivious trees. The dependence on $|D|$ in the former but not the latter can make a big difference on  practical feasibility for large-scale applications. For other cases, model-agnostic approximations (e.g., KernelSHAP \cite{LundbergLee} or Monte Carlo methods \cite{Strumbelj2014}) can be employed. 

Of course, the encoders must also be explainable. For instance, if the encoders are additive functions or involve pairwise interactions, their marginal Shapley values can often be computed analytically (see Section \ref{sec:additive_models}). Alternatively, if the encoders are trees (see Section \ref{sec::tree_rebalancing}), the aforementioned tree-based algorithms can be applied.

\begin{algorithm}[ht!]
\SetAlgoLined
\KwData{Model $f_*$, weight map $w$, initial parameter $\theta_{11}$, boundary conditions $\Theta$, training or holdout set $(X,Y, G)$, test set $(\bar{X},\bar{Y}, \bar{G})$
}
 \KwResult{Models $\{f(\cdot; \theta, f_*) \in \family(f_*, w), \theta\in \Theta\}$ constituting the efficient frontier of $\family(f_*, w)$ in \eqref{linearfam}. } 

{\bf Initialization parameters:} Fairness penalization parameters $\omega = \{ \omega_1, \dots, \omega_J \}$, learning rate $\alpha$, number of batch samples for estimating performance $n_{perf}$, number of batch samples for estimating bias $n_{bias}$, number of batches per epoch $n_{batch}$, and number $n_{epochs}$ of epochs of training for each $\omega_j$.\\
Pre-compute and store $f_*(X)$ and $w(X)$\\
Compute and store $\cL(\theta_{11}; X, Y)$ and $\B(\theta_{11}; X, G)$\\
\For{$j$ in $\{1,\dots,J\}$}
{ 
$\theta_{j1} := \text{argmin}_{\theta_{ji}}\, \cL(\theta_{ji}; X, Y) + \omega_j\cdot \B(\theta_{ji}; X, G)$\\
\For{$i$ in $\{1,\dots,n_{epochs}\}$}
{
 Compute and store $\cL(\theta_{ji}; X, Y)$ and $\B(\theta_{ji}; X)$\\
 $\theta_{j(i+1)} := \theta_{ji}$\\
 \For{$k$ in $\{1,\dots,n_{batches}\}$}
 {
    Produce $(X_{perf}, Y_{perf})$ by sampling $n_{perf}$ samples from $(X, Y)$\\
    Produce $(X_{bias}, G_{bias})$ by sampling $n_{bias}$ samples from $(X, G)$ for each $k\in G$\\
    Retrieve $f_*(X_{perf}), w(X_{perf})$\\
    Retrieve $f_*(X_{bias}), w(X_{bias})$\\
    Compute the gradient $d=\nabla_{\theta}[\cL(\theta; X_{perf}, Y_{perf}) + \omega_j\cdot \B(\theta; X_{bias}, G_{bias})] |_{\theta=\theta_{j(i+1)}}$\\
    Perform a gradient step, e.g., $\theta_{j(i+1)} \leftarrow \theta_{j(i+1)} - \alpha \cdot d$, such that $\theta_{j(i+1)}$ remains in $\Theta$
 }
}
}
Compute $(\theta_{ji}, \B(\theta_{ji};\bar{X}, \bar{G}), \cL(\theta_{ji};\bar{X},\bar{Y}))$ for $(j,i)\in \{1,\dots, J\}\times \{1,\dots, n_{epochs}\}$, giving collection $\mathcal{V}$.\\
 Compute the convex envelope of $\mathcal{V}$ and exclude the points that are not on the efficient frontier.

 \caption{Stochastic gradient descent for linear families with custom loss}\label{GDOalgo}
\end{algorithm}

Constructing the efficient frontier of the family \eqref{linearfam} amounts to solving the minimization problem \eqref{minfront}. Since any model from $\cF(f_*;w)$ is an adjustment of $f_*$ by a linear combination of encoders $\{w_j(\cdot;f_*)\}_{j=0}^m$, we propose employing stochastic gradient descent, as outlined in Algorithm \ref{GDOalgo},  where the map $\theta \mapsto \cL(\theta) + \omega \B(\theta)$ is approximated by appropriately designed differentiable estimators of bias metrics such as those in Section \ref{sec::biasapprox}.

This proposed SGD-based approach empowers us to learn highly complex demographically blind adjustments to our original model. Clearly, the selection of the weight maps $\{w_j\}_{j=0}^m$ is crucial for ensuring the explainability of post-processed models generated by this method. While the encoders may be constructed in a variety of ways, we present three particular approaches for producing families of fairer explainable models: corrections via additive models, tree-rebalancing (for tree ensembles), and finally explanation rebalancing.

\subsubsection{Corrections by additive models}\label{sec:additive_models}

First, we consider a simple case where the weight maps are fixed functions independent of the trained model. Specifically, let $\{q_j(t)\}_{j=0}^m$ be a collection of linearly independent functions defined on $\RR$, with $q_0(t) \equiv 1$. Define the corrective weights $w=\{w_0\} \cup \{w_{ij}\}$ by 
\begin{equation*}\label{additive-weights}
w_0(x)=1 \quad \text{and} \quad w_{ij}(x):=q_j(x_i), \quad i \in N:=\{1,\dots,n\}, \, j \in M:=\{1,\dots,m\}.
\end{equation*}
Then, any model in the family $\cF(f_*;w)$ has the representation
 \begin{equation*}\label{additive-weights-repr}
 f(x;\theta) = f_*(x) -  \bigg( \theta_0 + \sum_{i=1}^n \sum_{j=1}^m \theta_{ij}q_j(x_i) \bigg), \quad \theta:=\{\theta_0\} \cup \{\theta_{ij}\}.
\end{equation*}

Suppose $E(x;f)$, where we suppress the dependence on $X$, is a local model explainer defined for a family of ML models (assumed to be a vector space) that contains $f_*$ as well as the functions $\bar{q}_{ji}(x):=q_j(x_i)$, $i \in N, j \in M$. If $E$ is linear and centered, then the explanations of perturbed models can be obtained by
\begin{equation*}\label{additive-expl-gen}
E_i(x,f(\cdot;\theta))
=E_i(x; f_*) - \sum_{k=1}^n\sum_{j=1}^m \theta_{kj} E_i(x;\bar{q}_{jk}(x)) \,, \quad i \in N.
\end{equation*}

For example, the marginal Shapley value can be computed easily by leveraging the absence of predictor interactions:
\begin{equation*}\label{additive-corr-expl}
\varphi_i^{\ME}(x,f(\cdot;\theta))
=\varphi_i^{\ME}(x; f_*) - \sum_{j=1}^m \theta_{ij} \big( q_j(x_i)-\E[q_{j}(X_i)]\big)\,, \quad i \in N,
\end{equation*}
where we employed the null player property of the Shapley value.

\begin{remark}
\rm Note that, in practice, we may choose to fix some $\theta_{ij}$ to reduce the dimensionality of $\theta$.
\end{remark}

Note that the simplest bias correction approach, where $q_0(t)=1$, $q_1(t)=t$, corresponds to correcting the bias in our trained model scores using a function that is linear in the raw attributes of our dataset, that is, $f(x;\theta) = f_*(x) -  (\theta_0 + \sum_{i=1}^n \theta_{i} x_i)$, $\theta=(\theta_0, \dots \theta_n)$.  However, we may also employ nonlinear functions by letting $\{q_{j}(t)\}_{j=0}^m$ be the first $(m+1)$ basis polynomials of degree at most $m$. Such basis polynomials may be Legendre polynomials \cite{Legendre1785}, Bernstein polynomials \cite{Bernstein1912}, Chebyshev polynomials \cite{Chebyshev1854}, etc.  Another related approach involves replacing $\sum_{j=1}^m \theta_{ij} q_j(x_i)$ with $\cK_i(x_i; \theta_i)$, a single-variable neural network parametrized by weights $\theta_i$, or by using an explainable neural network based on additive models with interactions \cite{yang2021gami}. While outside the scope of this work, these approaches also yield explainable models that can be learned via SGD.

\subsubsection{Tree rebalancing}\label{sec::tree_rebalancing}

In many cases, linear combinations of fixed additive functions are not expressive enough to yield good efficient frontiers because of their modest predictive power \cite{Hastie2016}. Given this, it is worth exploring methods for constructing weight maps that include predictor interactions while remaining explainable.

In what follows, we design model-specific weights for tree ensembles. To this end, let us assume that a trained  model $f_*$ is a regressor (or a raw probability score) of the form 
\begin{equation*} \label{tree}
f_*(x)=\sum_{j=1}^m T_j(x)
\end{equation*}
with $\cR(f_*)=\{T_j\}_{j=1}^m$ being a collection of decision trees used in its representation. Let $A=\{j_1,j_2,\dots,j_r\} \subseteq \{1,\dots,m\}$ be a subset of tree indexes. Define the weights $w^{\small A}=\{w^A_k\}_{k=0}^r$ as follows:  
\begin{equation*}\label{weights-tree}
w_0^A(x;\cR(f_*)) \equiv 1 \quad \text{and} \quad w_k^A(x;\cR(f_*))=T_{j_k}(x), \,\quad k \in \{1,\dots,r\}.
\end{equation*}
Then, any model in the family $\cF(f_*;w^A)$ has the following representation:
 \begin{equation*}\label{weights-tree-repr}
 f(x;\theta) = f_*(x) -  \Big( \theta_0 + \sum_{k=1}^r \theta_{k} T_{j_k}(x) \Big) = \sum_{j \notin A} T_j(x) + \sum_{k=1}^r (1-\theta_k) T_{j_k}(x) - \theta_0,  \quad \theta=(\theta_0,\dots,\theta_k).
\end{equation*}

Suppose $E(x;f)$ is a local model explainer defined on a vector space of tree ensembles. If $E$ is linear and centered then, as $f(x;\theta)$ remains linear in the trees composing $f_*$, one has
\begin{equation}\label{additive-expl-tree-gen}
E_i(x, f(\cdot;\theta))= \sum_{j \notin A} E_i(x;T_j) + \sum_{k=1}^r (1-\theta_k) E_i(x;T_{j_k}), \quad i \in N.
\end{equation}

In particular, for marginal Shapley values we have
 \[
 \varphi_i^{\ME}(x, f(\cdot;\theta))= \sum_{j \notin A} \varphi^{\ME}_i(x;T_j) + \sum_{k=1}^r (1-\theta_k) \varphi_i^{\ME}(x;T_{j_k}),
 \]
where $\varphi^{\ME}_i(x;T_j)$ for each tree can be computed using the algorithms in \cite{filom2024,LundbergIntervTreeSHAP}.

We commonly seek $r\ll m$ to avoid overfitting when $m$ is large (e.g., $m=1000$). However, in practice, selecting which trees to include in $A$ is non-trivial. In search for a more strategic method of weight selection, we may generalize the above approach by considering weights $w=\{w_k\}_{k=0}^r$ that are linear combinations of trees:
\begin{equation}\label{weight-tree-comb}
w_0(x;\cR(f_*)) \equiv 1 \quad \text{and} \quad w_k(x;\cR(f_*))= \sum_{j=1}^m \alpha_{kj} T_j, \quad k \in \{1,\dots,r\},
\end{equation}
for some fixed coefficients $\alpha_{kj}$. Models incorporating such weights may be used without loss of explainability.

In this work, we select $a_{kj}$ using principal component analysis (PCA), which is compatible with formulation \eqref{weight-tree-comb}. Specifically, we design $\alpha_{kj}$ so that $w_k$ becomes the $k$-th most important principal component, in which case $\{w_k\}_{k=1}^r$ contains only the top $r$ principal components.

\begin{remark} \rm
The linearity property of $E$ can be relaxed. For example, if $E$ is homogeneous, centered, and tree-additive (i.e., defined as the sum of explanations of individual trees) then \eqref{additive-expl-tree-gen} holds.  This is the case for path-dependent TreeSHAP \cite{LundbergTreeSHAP}, where the underlying game is not linear with respect to models, but the explanation method -- which is based on the ensemble representation -- is tree-additive.
\end{remark}

\begin{remark} \rm
A drawback of non-sparse dimensionality reduction techniques such as PCA is that they require aggregating the outputs of each tree in the original model $f_*$. If one has a dataset with $10^6$ records and an ensemble with $10^3$ trees, employing PCA naively results in a $10^6\times 10^3$ matrix which may impose memory challenges. More sophisticated approaches, such as computing and aggregating tree outputs in batches, may mitigate this at the loss of parallelization. Alternatively, one may employ sparse PCA or some other sparse dimensionality reduction technique.
\end{remark}

\subsubsection{Explanation rebalancing}
\label{sec::explrebalancing}

In the event that one would like to incorporate predictor interactions without relying on the structure of the model, as we did with tree ensembles, the weights may be determined by explanation methods which, in principle, may be model-agnostic or model-specific.

As before, suppose $f_*$ is a trained model, a regressor, a classification score, or a raw probability score. First, we consider a special case. Let us define the weights $w=\{w_i\}_{i=0}^n$ to be marginal Shapley values: 
\begin{equation*}\label{explan-weights}
w_0(x;f_*) = 1 \quad \text{and} \quad w_{i}(x; f_*):=\varphi_i^{\ME}(x;f_*), \quad i \in N, \, x \in \RR^n.
\end{equation*}
Then, any model in the family $\cF(f_*;w)$ has the representation
 
\begin{equation}\label{mod-repr-game-reb}
\begin{aligned}
 f(x;\theta) &= f_*(x) -  \Big( \theta_0 + \sum_{i=1}^n \theta_i \varphi_i^{\ME}(x;f_*) \Big) 
 \\ &= \big(\E[f(X)]-\theta_0 \big) + \sum_{i=1}^n (1-\theta_i) \varphi_i^{\ME}(x;f_*), \quad \theta=(\theta_0,\dots,\theta_n) \in \RR^{n+1},
\end{aligned}
\end{equation}
where we use the efficiency of the marginal Shapley value, which reads $f(x)-\E[f(X)]=\sum_{i=1}^n \varphi_i^{\ME}(x;f)$.

Due to the predictor interactions incorporated into them, computing explanations for Shapley values $x \mapsto \varphi^{\ME}(x,f)$ themselves may be non-trivial even for explainable models.  However, model-specific algorithms for computing explanations may still be applicable. For example, \citep{filom2024} has established that a tree ensemble of oblivious trees is inherently explainable and its explanations coincide with marginal game values (such as the Shapley and Owen value)  which are constant in the regions corresponding to the leaves of oblivious trees. One practical consequence of this is that computing game values such as the marginal Shapley value for the map $x\mapsto\varphi^{\ME}(x;f)$ becomes feasible if $f$ is an ensemble of oblivious decision trees.

In the broader model agnostic case where Shapley rebalancing explanations are difficult to compute, one may heuristically define the explanations for the model $f(x;\theta)$ in the family $\cF(f_*;\{1\}\cup\{\varphi_i^{\ME}\})$  by setting $\tilde{E}_i(x;f(\cdot;\theta)):=(1-\theta_i)\varphi_i(x, f_*)$  which, by \eqref{mod-repr-game-reb} and the fact that $\E[\varphi_i^{\ME}(X;f_*)]=0$, $i \in N$, gives additivity:
\begin{equation}\label{additiv-game-reb}
\begin{aligned}
 f(x;\theta)-\E[f(x;\theta)] = \sum_{i=1}^n \tilde{E}_i(x;f(\cdot,\theta)), \quad \theta \in \RR^{n+1}.
\end{aligned}
\end{equation}

The above approach can be easily generalized. Suppose $E(x;f)$ is a local model explainer defined for a family of ML models, where we suppress the dependence on $X$. Given a trained model $f_*$, consider the family $\mathcal{F}(f_*,w)$ where $w_0=0$ and $w_i=E_i(\cdot;f_*)$ for $i \in N$. Suppose that $E$ is $P_X$-centrally additive. Suppose also  that it is $P_X$-centered, meaning $\E[E_i(X;f)]=0$, $i \in N$. Then, for any $f(\cdot;\theta) \in \mathcal{F}(f_*,w)$, the additivity property \eqref{additiv-game-reb} holds with $\tilde{E}(x;f(\cdot,\theta)):=(1-\theta_i)E(x;f_*)$. Thus, the heuristic explainer $\tilde{E}$ is $P_X$-centrally additive on the family of models $\mathcal{F}(f_*,w)$. We also note that by design, $\tilde{E}$ is linear on $\mathcal{F}(f_*,w)$.

However, unless predictor interactions are crucial for generating good frontiers, it maybe preferable to use the bias correction methodology via additive functions discussed earlier rather than resorting to a heuristic. Finally, we note that other efficient game values, such as the Owen value, can be used for the weights.

\begin{remark} \rm
The approaches we presented above are not mutually exclusive and may be combined without affecting explainability.
\end{remark}


\section{Experiments}\label{sec::experiments}

In this section, we examine how different bias mitigation methods fare on various synthetic and real world datasets. Specifically, we consider three distinct post-processing strategies: predictor rescaling based on \cite{Miroshnikov2020b}, an optimal transport projection method based on \cite{miroshnikovpatent} adapted for explainability, and the new methods we proposed in Section \ref{sec::methods_perturb} based on gradient descent. We implement these strategies in the same way across all datasets.

We test the gradient-based methods we introduced in Section \ref{sec::methods_perturb}: additive model correction, tree rebalancing, and explanation rebalancing. For additive model correction, we treat all features as numerical and remove missing values by imputing the mean. For tree rebalancing, we employ forty weight values based on equation \eqref{weight-tree-comb} with fixed coefficients $a_{kj}$ determined by PCA to avoid over-parametrization. For explanation rebalancing, we use CatBoost's built-in marginal Shapley value implementation and thus refer to this particular methodology as Shapley rebalancing. Using the stochastic gradient descent strategy of algorithm \ref{GDOalgo}, we minimize a cross-entropy loss subject to a fairness penalization based on \eqref{gen_bias}. Specifically, the penalization uses the cost function $c(a,b)=h(a-b)=(a-b)^2$ that  corresponds to the energy distance bias \eqref{energy_bias_mu} and ensures that both the estimators \eqref{int_approx_est} and \eqref{E-statistic-mu} are unbiased. However, our results are still presented in the context of the Wasserstein metric (i.e. $h(\cdot)=|\cdot|$) due to its practical interpretation and relationship to prior literature (e.g. \cite{Miroshnikov2020}). For further details on our gradient descent approach and custom loss function, see Appendix \ref{app::perturbation}.

To compare with the method of \cite{Miroshnikov2020b}, we perform  rescaling of numerical predictors $X_i$ by substituting them with transformed predictors $X_i'(\theta_i)$ defined by linear transformations of the form $X_i'(\theta_i)=\theta_i(X_i-\bar{X_i})+\bar{X_i}$ where $\bar{X}$ denotes the mean value of $X_i$. We attempt to solve \eqref{minfront} for these parameters using two different strategies, random search (1150 iterations) and Bayesian search (1050 iterations with a prior based on 100 randomly selected models). Due to the dimensionality limitations of these optimization methods, we only rescale a collection of predictors with notable bias explanations using the method of \cite{Miroshnikov2020}. For more details on these predictors and the handling of categorical variables using this method, see appendix \ref{app::predictorrescaling}.

We also compare with optimal transport projection \cite{miroshnikovpatent}, which is implemented by training CatBoost models without access to the protected attribute to approximate the fair models described in the works of \citep{Loubes2020, Chzhen2020, Chzhen2022, Kwegyir-Aggrey2023}, which explicitly rely on the protected attribute. We generate fifteen models of different bias and performance levels by linear interpolation between the original trained models and the new CatBoost models. For more details about this approach, see appendix appendix \ref{app::opttransp}.

In this section, we investigate how our methods perform on a range of synthetic and real world datasets representing classification tasks. We implement our strategies in the same way across all experiments. Predictor rescaling is implemented using both random search (1150 iterations) and Bayesian search (1050 iterations after generating a prior of 100 random models) so our predictor rescaling frontiers reflect the combined result of both search strategies. The methods we proposed in Section \ref{sec::methods_perturb} -- additive model correction, tree rebalancing, and Shapley rebalancing -- are implemented using stochastic gradient descent as described in Algorithm \ref{GDOalgo}. Optimal transport projection is implemented as described in appendix \ref{app::opttransp} using CatBoost classification models with fifteen evenly spaced values of $\sqrt{\alpha} \in [0, 5]$. See appendices \ref{app::predictorrescaling}, \ref{app::perturbation}, \ref{app::opttransp} for further implementation details on the predictor rescaling method, SGD-based methods (i.e. additive model correction, tree rebalancing, and Shapley rebalancing), and the optimal transport based mitigation method respectively.

\subsection{Synthetic datasets}

We begin by studying our methods on the synthetic examples introduced by \citep{Miroshnikov2020b}. In \eqref{realisticmod}, predictors may contribute positively or negatively to the classifier bias of $P(Y=1 | X)$ depending on the selected decision threshold. However, the positive contributions dominate resulting in a true model which only has positive bias. 

\begin{equation}\label{realisticmod}\tag{M1}
\begin{aligned}
&\mu=5, \quad a=\tfrac{1}{20}(10,-4,16,1,-3)\\
&G\sim Bernoulli(0.5)\\
&X_1 | G\sim N(\mu-a_1 (1-G), 0.5+G ), \quad X_2 | G \sim N(\mu-a_2 (1-G), 1 ) \\
&X_3 | G \sim N(\mu-a_3 (1-G), 1 ),     \quad X_4 | G \sim N(\mu-a_4 (1-G), 1-0.5 G ) \\
&X_5 | G \sim N(\mu-a_5 (1-G),1-0.75 G ) \\
&Y | X \sim Bernoulli(g(X)), \quad g(X)=\sigma(f(X))=\P(Y=1|X),\quad f(X)=2(\textstyle{\sum_{i}} X_i-24.5).
\end{aligned}
\end{equation}

We also introduce another data generating model, \eqref{realisticmodasym}, which has two predictors $X_1, X_4$ that positively impact bias at some decision thresholds and negatively impact it at others. As a result, symmetrically compressing these predictors as is done in predictor rescaling may have an unclear impact on overall model bias. This may potentially also have consequences for other strategies that may be viewed as symmetrically compressing the impact of a predictor, such as Shapley rebalancing.
\begin{equation}\label{realisticmodasym}\tag{M2}
\begin{aligned}
&\mu=5, \quad a=\tfrac{1}{10}(2.5,1.0,4.0,-0.25,0.75)\\
&G \sim Bernoulli(0.5)\\
&X_1 | G\sim N(\mu-a_1 (1-G), 0.5+G \cdot 0.75), \quad X_2 | G\sim N(\mu-a_2 (1-G), 1 ) \\
&X_3 | G \sim N(\mu-a_3 (1-G), 1 ),     \quad X_4 | G\sim N(\mu-a_4 (1-G), 1-0.75 G ) \\
&X_5 | G\sim N(\mu-a_5 (1-G), 1) \\
&Y | X \sim Bernoulli(g(X)), \quad g(X)=\sigma(f(X))=\P(Y=1|X),\quad f(X)=2(\textstyle{\sum_{i}} X_i-24.5).
\end{aligned}
\end{equation}

To test our bias mitigation strategies, we generate 20,000 records from the distributions defined by these data-generating models and split them equally among a training dataset and testing dataset. Then we use the training dataset in both training a CatBoost model to estimate $Y | X$ and in applying our strategies to mitigate the bias of those CatBoost models. Table \ref{tbl::synthdatasets} describes some aspects of the datasets we generate along with the $W_1$ biases of our trained CatBoost models.

\begin{table} 
\centering
\begin{tabular}{@{}llllllll@{}}
\toprule
Dataset & $d_{features}$ & $n_{train+test}$ & $n_{unprot}$ & $n_{prot}$ & $W_1$ Bias \\ \midrule
Data Model 1 & 5        & 20000 & 9954 & 10046           & 0.1746     \\
Data Model 2 &  5        & 20000 & 9954 & 10046          & 0.1340     \\
\bottomrule
\end{tabular}
\caption{Summary table for synthetic datasets. $d_{features}$ describes the number of features present in each dataset. $n_{train+test}$, $n_{unprot}$, and $n_{prot}$ describe the number of observations in the dataset, the number of observations of the unprotected class, and the number of observations in the protected class respectively. $W_1$ bias is reported for the CatBoost models trained on these datasets to predict $Y$ which we employ in testing our mitigation techniques.}\label{tbl::synthdatasets}
\end{table}

\begin{figure}[ht!]
\centering
\includegraphics[width=0.66\textwidth]{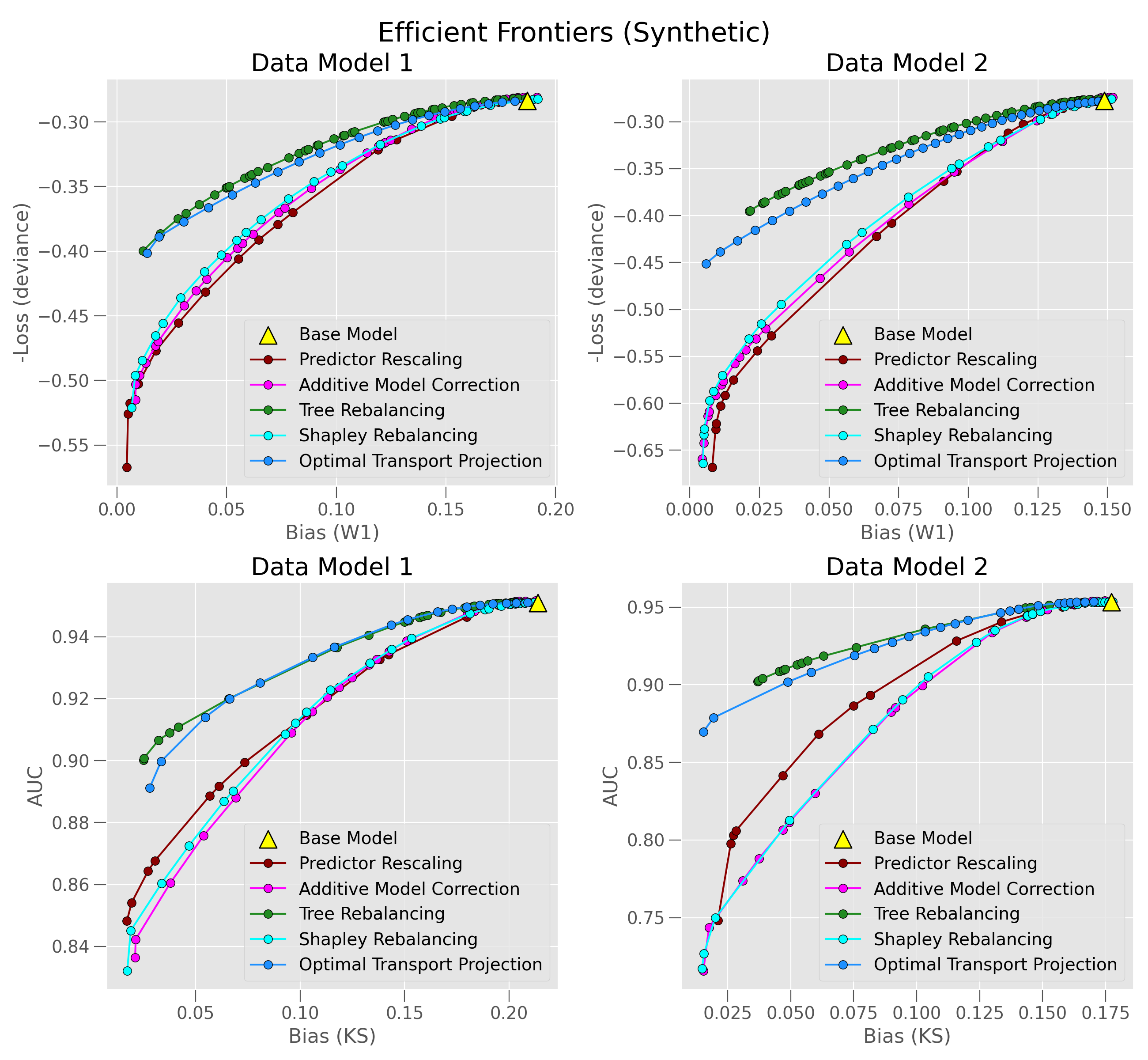}
\caption{\footnotesize Efficient frontiers for data model 1 and data model 2. All results are presented on their respective test datasets.}\label{fig::BPsynthetic}
\end{figure}

The bias performance frontiers resulting from applying our strategies to these data generating models are shown in Figure \ref{fig::BPsynthetic}. We see that tree rebalancing and optimal transport projection produce models with the best bias performance trade-offs. Additive model correction, Shapley rebalancing, and predictor rescaling perform similarly and yield less optimal bias performance trade-offs. Note that the relative performance of bias mitigation strategies is similar across data-generating models (\ref{realisticmod}, \ref{realisticmodasym}) and across bias-performance metric pairings (binary cross-entropy vs Wasserstein-1 bias, AUC vs Kolmogorov-Smirnov bias).

The relative performances of methods on these synthetic datasets can be understood in the context of free parameters. Tree rebalancing has a free parameter for each of the forty principal components that rebalancing is applied to and optimal transport projection learns a new CatBoost model with a flexible number of parameters. In contrast, additive model correction, Shapley rebalancing, and predictor rescaling only have five parameters, one for each predictor. The similarities in the frontiers for additive model correction, Shapley rebalancing, and predictor rescaling are also no coincidence. When one is mitigating models linear in predictors, the model family explored by Shapley rebalancing is equivalent to the model family produced by predictor rescaling, and also equivalent to the model family produced by additive terms linear in the raw attributes assuming no limitations on mitigation parameters. In our synthetic examples, the true score $f(x)$ is linear in predictors, so this correspondence is approximately true when mitigating trained models.

\subsection{Real world datasets}\label{sec::realworld}
We also examine the efficient frontiers produced by our strategies on real world datasets common in the fairness literature: UCI Adult, UCI Bank Marketing, and COMPAS. These datasets contain a range of protected attributes (gender, age, race), prediction tasks, levels of data imbalance, and may yield trained models with relatively high $W_1$ bias. Summary information for these datasets is provided in Table \ref{tbl::datasets}.

\begin{table} 
\centering
\begin{tabular}{@{}llllllll@{}}
\toprule
Dataset      & $G$ & $Y$     & $d_{features}$ & $n_{train+test}$ & $n_{unprot}$ & $n_{prot}$ & $W_1$ Bias \\ \midrule
UCI Adult \citep{adult_2, Vogel2021, Becker2024}        & Gender              & Income       & 12       & 48842  & 32650 & 16192      & 0.1841     \\
UCI Bank Marketing \citep{bank_marketing_222, Jiang2020Jul}   & Age                 & Subscription & 19       & 41188 & 38927 & 2261       & 0.1903     \\ 
COMPAS \citep{compas, zafar2017, Vogel2021, Becker2024}      & Race                & Risk Score   & 5        & 6172 & 2103 & 3175           & 0.1709     \\
\bottomrule
\end{tabular}
\caption{Summary table for datasets used in experiments. $d_{features}$ describes the number of features present in each dataset. $n_{train+test}$, $n_{unprot}$, and $n_{prot}$ describe the number of observations in the dataset, the number of observations of the unprotected class, and the number of observations of the protected class respectively. $W_1$ bias is reported for the CatBoost models trained on these datasets to predict $Y$ which we employ in testing our mitigation techniques.}\label{tbl::datasets}
\end{table}

At a high-level, the UCI Adult dataset contains demographic and work-related information about individuals along with information about income. As in \citep{Jiang2020Jul}, we build a model using this dataset to predict whether an individual's annual income exceeds \$50,000 and then attempt to mitigate this model's gender bias. The UCI Bank Marketing marketing dataset also describes individuals and we employ it to build a model that predicts whether individuals subscribe to a term deposit. We then attempt to mitigate this models' age bias. Lastly, COMPAS is a recidivism prediction dataset. We employ it to build a model that predicts whether individuals are classified as low or medium/high risk for recidivism by the COMPAS algorithm. In this case, we are attempting to mitigate this model's race bias. For all mitigation exercises, we split the datasets 50/50 into train and test sets, with model building and mitigation performed using the training dataset. For more details about these datasets, feature preprocessing steps, and filtration procedures, see appendix \ref{app::datasets}. 

Note that, for predictor rescaling, we consider rescaling all five numerical features in the UCI Adult dataset (seven others are categorical), eight of the thirteen numerical features in the UCI Bank Marketing dataset (six others are categorical) and all five features in the COMPAS dataset. The purpose of limiting the number of features under consideration is in part to reduce the dimensionality of the random/Bayesian search problem. For more details on the predictors considered during rescaling, see appendix \ref{app::predictorrescaling}. For an empirical demonstration of what occurs when predictors are not restricted, see appendix \ref{app:beyondBPEF}.

\begin{figure}[ht!]
\centering
\includegraphics[width=\textwidth]{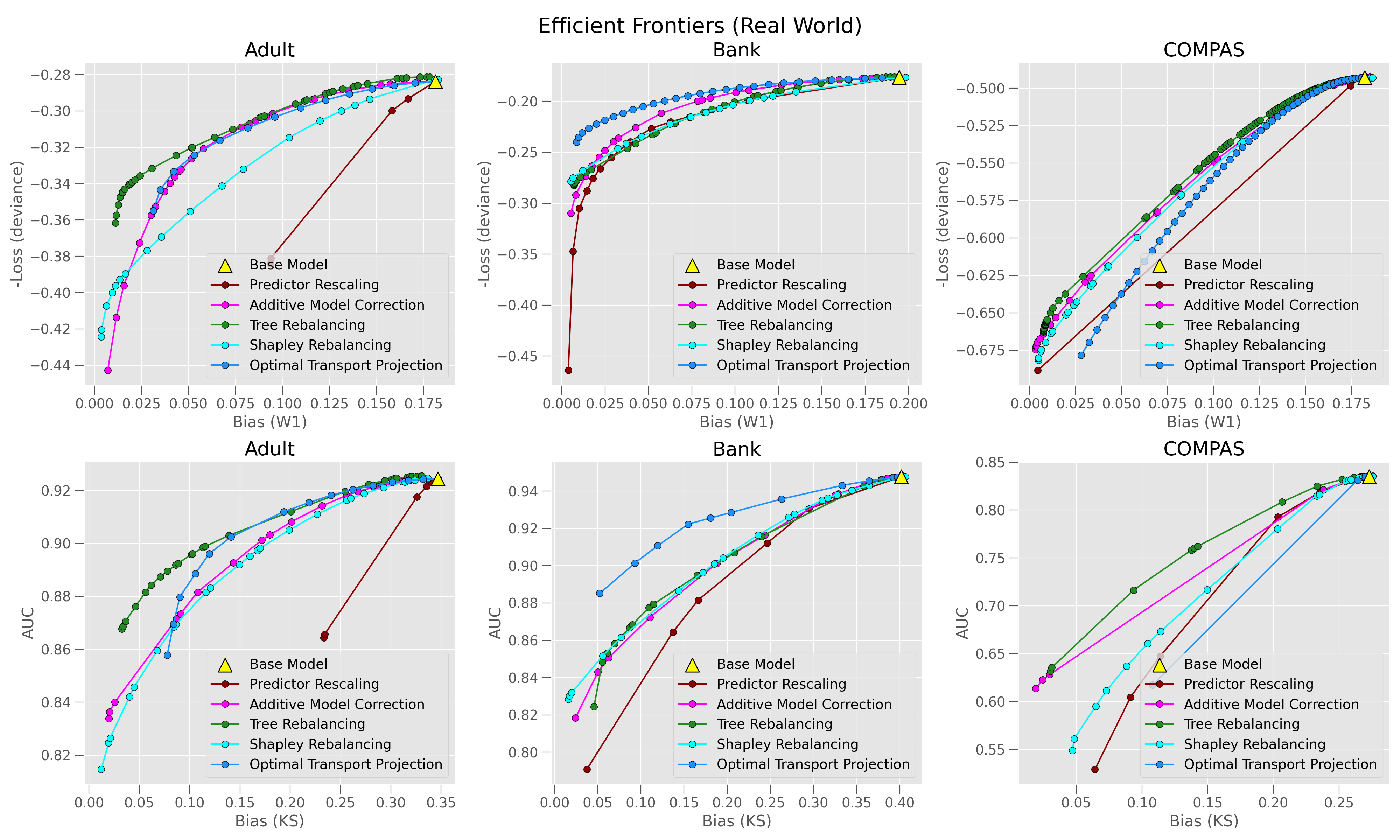}
\caption{Efficient frontiers for UCI Adult, UCI Bank Marketing, and COMPAS datasets, evaluated on their test datasets.}
\label{fig::BPrealworld}
\end{figure}

The results of performing different mitigation strategies on these datasets are shown in Figure \ref{fig::BPrealworld}. Note that no bias mitigation method dominates, but some strategies can perform much better than others in certain contexts. Furthermore, performance of a mitigation approach may depend on the bias performance frontier being targeted. For example, on the UCI Adult dataset, the additive model correction method performs very well (comparably to optimal transport projection) on the cross-entropy vs Wasserstein-1 frontier but more poorly (comparably to Shapley rebalancing) on the AUC vs KS frontier. To get a general sense of how methods perform, we holistically consider both frontiers in this work.

On the UCI Adult dataset, which has a moderate number of features and many observations from both protected and unprotected classes, strategies optimizing higher dimensional spaces are at an advantage. Optimal transport projection and tree rebalancing perform the best, followed by additive model correction, Shapley rebalancing, and then predictor rescaling. Predictor rescaling likely performs more poorly here than on the synthetic datasets because it optimizes a lower dimensional space of five predictors, while Shapley rebalancing and additive model corrections each adjust model scores using all twelve (we apply this restriction because Bayesian / random search struggle using all predictors; see Appendix \ref{app:beyondBPEF}). Other than this difference, the results are similar to those on the synthetic datasets.

In contrast, UCI Bank Marketing has a limited number of observations from the protected class yet more features than UCI Adult. Here, optimal transport projection outperforms the rest of the post-processing methods considerably and remaining differences in frontiers are fairly minor. Since optimal transport projection is the only method that does not explicitly depend on $Y$, its performance may be due to a lack of overfitting relative to the other methods.

Lastly, we may consider COMPAS, which has a low number of features and few observations. In this case, all methods perform similarly. Note that, due to its low number of predictors, COMPAS is the only dataset where predictor rescaling employs all features in the dataset and therefore may compete with Shapley rebalancing and additive model correction. In general however, Shapley rebalancing, predictor rescaling, and additive model correction will no longer have the correspondence seen on the synthetic examples for two reasons: the trained models may not be linear, and Shapley rescaling can more naturally handle categorical variables than predictor rescaling.

\begin{remark}[Addressing overfitting in SGD methods]
Mitigation methods employing high-dimensional optimization may overfit to the response variable $Y$. In such cases, removing the explicit use of $Y$ in \eqref{minfront} may improve test dataset performance. Following Hinton's work on knowledge distillation in the classification setting \cite{hinton2015distilling}, we propose substituting $\cL(\theta)=\E[L( \sigma \circ f(X;\theta),Y)]$ with an alternative loss term $\E[\tilde{L}(\sigma \circ f(X;\theta), \sigma \circ f_*(X))]$, where $\tilde{L}$ is defined as:
\begin{align*}
\tilde{L}(p,q) :&= D_{KL}(\text{Bern}(p) || \text{Bern}(q))=p\log\left(\frac{p}{q}\right) + (1-p)\log\left(\frac{1-p}{1-q}\right)\,,
\end{align*}
where $D_{KL}$ denotes the Kullback–Leibler divergence and $\text{Bern}(p)$ the Bernoulli distribution.
\end{remark}

\begin{appendices}
\section*{Appendix}

\section{Experimental details}\label{sec::experimental_details} 
This section describes the details of the bias mitigation experiments presented in the main body of the paper. In appendix \ref{app::datasets}, we review the datasets used; in appendix \ref{app::models}, we describe the model development procedure we employed to then showcase our bias mitigation approach; and in appendices \ref{app::predictorrescaling}, \ref{app::perturbation}, and \ref{app::opttransp} we present the predictor rescaling, perturbation, and explainable optimal transport projection methods respectively.

\subsection{Datasets}
\label{app::datasets}



\paragraph{UCI Adult.} The UCI Adult dataset includes five numerical variables (Age, education-num, capital gain, capital loss, hours per week) and seven categorical variables (workclass, education, marital-status, occupation, relationship, race, native-country) for a total of twelve independent variables. In addition, UCI Adult includes a numerical dependent variable (income) and gender information (male and female) for each record. We encode the categorical variables  with ordinal encoding and we binarize income based on whether it exceeds \$50,000.

The task on the UCI Adult dataset is to mitigate gender bias in a machine-learning model trained to classify records as having income in excess of \$50,000. To do this, we merge the default train and test datasets associated with UCI Adult together and randomly split them 50/50 into a new train and test dataset. The machine-learning model is trained using the new training dataset, although the early-stopping procedure employs the new test dataset. Bias mitigation techniques are applied using only the training dataset.


\paragraph{UCI Bank Marketing.} The UCI Bank Marketing dataset includes thirteen numerical variables (default, housing, loan, duration, campaign, pdays, previous, poutcome, emp.var.rate, cons.price.idx, cons.conf.idx, euribor3m, nr.employed) and six categorical variables (education, month, day\_of\_week, job, marital, contact). We convert education to a numerical variable based on the length of schooling suggested by its categories. Similarly, we convert month and day\_of\_week to numerical variables with month going from zero to eleven (January through December) and day\_of\_week going from zero to four (Monday through Friday). We encode the rest of the categorical variables using ordinal encoding. Furthermore, we represent categories like unknown or non-existent as missing to leverage CatBoost's internal handling of missing values. The dependent variable is a yes/no classification reflecting whether marketing calls to a client yielded a subscription. UCI Adult also includes age information which is treated as sensitive demographic information.

The task on the UCI Bank Marketing dataset is to mitigate age bias in a machine-learning model trained to predict subscriptions. We base this on two age classes, one for ages in $[25, 60)$ and one for all other ages. To do this,we randomly split the UCI Bank Marketing dataset 50/50 into a train and test dataset. The machine-learning model is trained using the training dataset and the test dataset is used for early-stopping. Bias mitigation techniques are similarly applied using the training dataset.


\paragraph{COMPAS.} The COMPAS dataset includes three numerical variables (priors\_count, two\_year\_recid) and three categorical variables (c\_charge\_degree, sex, age). Age is encoded from zero to two in order of youngest to oldest age group. Sex and c\_charge\_degree are ordinal encoded. The dependent variable describes the risk classification of the COMPAS algorithm, which we binarize as zero if low and one otherwise. COMPAS also includes race information which is treated as sensitive demographic information. Following \citet{zafar2017}, we only include records when days\_b\_screening\_arrest $\in [-30, 30]$, is\_recid $\neq -1$, c\_charge\_degree $\neq 0$ and the risk score is available.

The task on the COMPAS dataset is to mitigate black/white racial bias in a machine-learning model trained to predict the COMPAS risk classification. To do this, we randomly split the filtered COMPAS dataset 50/50 into a train and test dataset. The machine-learning model is trained using the training dataset and the test dataset is used for early-stopping. Bias mitigation techniques are similarly applied using the training dataset.

\subsection{Model construction for experiments}
\label{app::models}

In order to generate realistic models for the purpose of testing mitigation techniques, a simple model building pipeline was implemented for all experiments. First the relevant dataset was standardized and split 50/50 into train and test datasets. Then 100 rounds of random hyperparameter tuning were performed using CatBoost models with hyperparameter combinations described by Table \ref{table::hyperparameters}.

\begin{table}[ht!]
\centering
\begin{tabular}{@{}ll@{}}
\toprule
Hyperparameter & Value Range \\ \midrule
depth & $\in \{3,4,5,6,7,8\}$ \\
max iterations & $\in \{1000\}$ \\
learning rate & $\in \{0.005, 0.01, 0.04, 0.08\}$ \\
bagging temperature & $\in \{0.5, 1.0, 2.0, 4.0, 8.0\}$ \\
l2 leaf reg & $\in \{1, 2, 4, 8, 16, 32\}$ \\
random strength & $\in \{0.5, 1.0, 2.0, 4.0, 8.0\}$ \\
min data in leaf & $\in \{2, 4, 8, 16, 32\}$ \\
early stopping rounds & $\in \{1, 2, 4, 8\}$ \\ \bottomrule
\end{tabular}\caption{Hyperparameter combinations sampled in the model development procedure for our experiments during the random hyperparameter tuning phase.}\label{table::hyperparameters}
\end{table}
Following this, the model with the lowest test loss of all models with train/test loss percent differences below 10\% was selected. If no models with train/test loss percent differences below 10\% existed, the model with lowest train/test loss percent difference was selected. These final models were stored and bias mitigation was attempted on each using all strategies to ensure apples-to-apples comparisons.


\subsection{Predictor rescaling methodology}\label{app::predictorrescaling}

Predictor rescaling as described in \citet{Miroshnikov2020b} was implemented with $n_{prior}=100$, $n_{bo}=50$, and $\{\omega_j\}_{j=1}^{21} = \{\tfrac{1}{2}\cdot (j-1)\}_{j=1}^{21}$; see Algorithm \ref{BHPSalgo} below, which originates from that paper. Bayesian optimization parameters were given by $\kappa=1.5$ and $\xi=0.0$. We also implement predictor rescaling using $1150$ iterations of random search. The results of predictor rescaling we present combine the best of both these methods. In this work, we make use of a linear compressive family with transformations 
\[T(t,a,t^{*})=a(t-t^{*}) + t^{*}\]
for numerical variables. We accommodate interventions on categorical variables by subtracting weights $w_k(x)$ with the intent to cancel out their contributions to the model output. Thus, the final post-processed model is of the form
\[\bar{f}(x; \alpha, x^{*}_M, \beta, w^{*}_N) = f(T(x_M,a,x^{*}_M),x_{-M})+\sum \beta_k (w_k(x)-\bar{w}_k)\,\]
where $w_k(x)$ is some way of determining the contribution of the $k$-th feature to the model output and $\bar{w_k}$ denotes a baseline value. In our experimental implementation we let $w_k(x)=f(x)-f(x_{-\{k\}}, \E[{X}_{\{k\}}])$ and $\bar{w}_k=\E[w_k(X)]$ for categorical predictor indices $k$. We let $\alpha_i \in [0,3]$, $x_i^*\in [X_i^{*}+0.05(X_i^{*}-\min(X_i)),\, X_i^{*}+0.05(\max(X_i)-X_i^{*})]$, and $\beta_k \in [-1.5,1.5]$. To further reduce the search space, we limit the number of compressive transformations and weight adjustments to subsets of predictors following the method of \citep{Miroshnikov2020}. These predictors, along with their bias ranks, are provided below:

\begin{itemize}
\item[$\bullet$] \textbf{(UCI Adult)} The top five most biased predictors: relationship (1st), marital-status (2nd), hours per week (3rd), Age (4th), capital gain (5th).
\item[$\bullet$] \textbf{(UCI Bank Marketing)} The top eight most biased predictors \textemdash nr.employed (1st), euribor3m (2nd), month (3rd), cons.price.idx (4th), duration (5th), emp.var.rate (6th), pdays (7th) , cons.conf.idx (8th).
\item[$\bullet$] \textbf{(COMPAS)} All predictors: priors\_count (1st), age (2nd), two\_year\_recid (3rd), c\_charge\_degree (4th), sex (5th).
\end{itemize}

\begin{algorithm}[ht!]
\SetAlgoLined 
\KwData{Model $f$, training or holdout set $(X,Y)$, test set $(\bar{X},\bar{Y})$, the set $M$ of bias-impactful predictors.
}
 \KwResult{Models on the efficient frontier of the parametrized family of models $\mathcal{F}$} 

{\bf Initialization parameters:} the number $n_{prior}$ of random points  $\gamma = (\alpha, x_{M}^*)$, the prior $P_{prior}(d\gamma)$, fairness penalization parameters $\omega = \{ \omega_1, \dots, \omega_J \}$, the number $n_{bo}$ of Bayesian steps  for each $\omega_j$.\\
Sample $\{\gamma_i\}_{i=1}^{n_{prior}}$ from $P_{prior}(d\gamma)$\\
\For{$i$ in $\{1,\dots,n_{prior}\}$}
{
$loss(\gamma_i;X,Y):=\E[\mathcal{L}(X,\bar{f}(X;\gamma_i))]$, $\bar{f} \in \family(f)$.\\
$bias(\gamma_i;X):=\Bias_{W_1}(\bar{f}(X;\gamma_i)|G)$; see definition \ref{modbias}.
}

\For{$j$ in $\{1,\dots,J\}$}
{ 
\For{$i$ in $\{1,\dots,n_{prior}\}$}
{
 $L(\gamma_i,\omega_j) := loss(\gamma_i; X,Y) + \omega_j \cdot bias(\gamma_i;X)$
}
Pass $\{ \gamma_i, L(\gamma_i,\omega_j)\}_{i=1}^{n_{prior}}$ to the Bayesian optimizer that seeks to minimize $L(\cdot,\omega_j)$.

Perform $n_{bo}$ iterations of Bayesian optimization which produces $\{\gamma_{t,j} \}_{t=1}^{n_{bo}}$.
}
Compute $(\gamma, bias(\gamma;\bar{X}), loss(\gamma;\bar{X},\bar{Y}))$ for $\gamma \in \{\gamma_i\} \cup \{\gamma_{t,j}\}$, giving a collection $\mathcal{V}$.\\
 Compute the convex envelope of $\mathcal{V}$ and exclude the points that are not on the efficient frontier.
 \caption{Efficient frontier reconstruction using Bayesian optimization}\label{BHPSalgo}
\end{algorithm}

\subsection{Perturbation methodology}
\label{app::perturbation}
Algorithm \ref{GDOalgo} was implemented using optimization objectives of the form $(1-\omega_j)\cL + \omega_j\B$. We set $\cL$ to be binary cross-entropy and $\B$ to be based on \eqref{int_approx_est} with $r_s(z)=\sigma(20z)$, $h(z)=z^2$, $\rho(t)=1$, and $\Delta t=1/129$ with some adjustments. Specifically, we improve the discrete estimator using the trapezoidal rule alluded to in Remark \ref{rem::higherorder} and modify the terms we sum to yield an unbiased estimator. This is possible due to this metric's connection with the energy distance; see (\ref{energy_bias_mu}). Furthermore, we let $n_{perf}=n_{bias} = 1024$, learning rate $\alpha=0.01$, and $n_{epochs}=20$. We also let $w_j = (C/20)\cdot j$ for $j\in \{0, 1,\dots, 20\}$ for an appropriate scaling constant $C$ (typically either one or the ratio of binary cross-entropy to the $\B$ bias in the original model). For tree rebalancing, we apply PCA to the dataset $(T_1(X),\dots, T_n(X))$ with $T_i$ being trees in the GBDT targeted for mitigation and $X$ being the dataset. Then rebalancing was done using the most important 40 principal components.

\subsection{Explainable optimal transport projection methodology}
\label{app::opttransp}
The application of optimal transport to bias mitigation procedures has been extensively studied. Bias metrics inspired by the Wasserstein distance have been discussed in \citep{Miroshnikov2020, Becker2024} and model training using penalized losses based on these metrics has been described in \citep{Jiang2020Jul}. Methods for de-biasing datasets using optimal transport have been proposed by \citep{Feldman2015, johndrow2019, delBarrio} and similar techniques have been proposed for post-processing model predictions by \citep{Loubes2020, Chzhen2020, Chzhen2022, Kwegyir-Aggrey2023}. \citep{miroshnikovpatent} proposes eliminating direct use of demographic information through projection as follows:
\[\tilde f(x)=\mathbb{E}[\bar f(X,G) | X=x] = \sum_{k=0}^{K-1} \bar f(x,k)\cdot \mathbb{P}(G=k|X=x)\,,\]
where $\bar f(x, k)$ is any fair model using demographic information such as one produced using \citep{Loubes2020, Chzhen2020, Chzhen2022, Kwegyir-Aggrey2023} and $\mathbb{P}(G=k|X=x)$ are regressors trained on $(X,G)$. However, this approach involves non-linear function compositions and multiplication by regressors $\mathbb{P}(G=k | X=x)$ and thus may be difficult to explain.

To facilitate explainability, we directly estimate $\tilde f(x)=\mathbb{E}[\bar f(X,G) | X=x]$ using an explainable ML model. This may be done in several ways however, for binary classification tasks, we do so by training an explainable model $\tilde f(x)$ on the following constructed dataset
\[(X',Y',W'):=(X, 0, 1-\bar f(X, G)) \oplus (X, 1, \bar f(X, G))\,,\]
where $\oplus$ indicates dataset concatenation and $X',Y',W'$ correspond to predictors, labels, and sample weights. Clearly, $\E[\bar f(X,G) | X=x] = \E[W'\cdot Y' | X'=x]$, so $\tilde f$ is an explainable projection of $\bar f$.

We may now create a simple model family based on projection of fair models produced using optimal transport methods. The fair optimal transport model given by \citep{Chzhen2020} and projection following \citep{miroshnikovpatent} are defined by the following maps,
\begin{equation*}
f\mapsto\bar f(x,k;f):=\left (\sum_{k'=0}^{K-1} p_{k'} F^{[-1]}_{k'} \right)\circ F_{k}(f(x));\quad f\mapsto\tilde f(\cdot;\bar f)= \E[\bar f(X,G) | X=x]],
\end{equation*}
where $p_k=\mathbb{P}(G=k)$. The linear family \eqref{linearfam} can now be constructed using the optimal transport projection weight $w(\cdot; f)=f(\cdot) - \tilde f(\cdot;\bar f(\cdot,\cdot;f))$ to yield a one-parameter model family. Note that any model in this family $\cF(f_*;w)$ has the representation
\begin{equation*}
f(x;\theta) = f_*(x) - \theta ( f_*(x)-\tilde f(x)) = (1-\theta)f_*(x) + \theta \tilde f(x)\,,
\end{equation*}
which happens to include explainable projections of models $f_{\alpha}(x)=(1-\alpha)f_*(x) + \alpha \bar f(x,k,f_*)$. The partially repaired models $f_{\alpha}$ are discussed in \cite{Chzhen2020, Kwegyir-Aggrey2023} as yielding good bias performance trade-offs.

In our experiments, we estimate $\tilde f(\cdot;f)$ by training explainable CatBoost models. CatBoost models were trained with depth 8, learning rate $0.02$, and 10 early stopping rounds with $1000$ maximum iterations. Early stopping was performed based on the test dataset.

\subsection{Impact of dimensionality on predictor rescaling methods}
\label{app:beyondBPEF}
While Bayesian search and random search face challenges when optimizing in higher-dimensional parameter spaces \citep{frazier2018}, higher-dimensional parameter spaces may also encompass models with bias performance trade-offs superior to those available in lower-dimensional spaces. Thus, the empirical impact of constraining the predictors being rescaled in \cite{Miroshnikov2020b} is worthy of investigation. 

This question is the subject of Figure \ref{fig::BayesRealworld}. The first row displays results from the real world experiments presented in the main body of our paper. The second row displays new results from analogous experiments performed by allowing rescaling of all predictors (Note that we perform predictor rescaling of all five COMPAS predictors in our original experiments). To better understand how dimensionality impacts random/Bayesian search, results from all probed models are shown in addition to associated efficient frontiers. The more frequently the search procedure (i.e. Bayesian search, random search) finds models near the frontier, the more confident one can be that the space is being explored well.

On UCI Adult, using all predictors visibly reduces efficient frontier performance and broadly lowers the general quality of models probed by Bayesian search. On UCI Bank Marketing, results are somewhat different: Using all predictors does not meaningfully impact efficient frontier importance but does result in more models near the frontier in the low bias region. 

Given these observations, rescaling a smaller number of important features is, overall, a better bias mitigation approach for the datasets in this article. This justifies the approach to predictor rescaling presented in the main body of this text. Note however that, in general, the number of features appropriate for use in the predictor rescaling method is dataset dependent.

\begin{figure}[ht!]
\centering
\includegraphics[width=\textwidth]{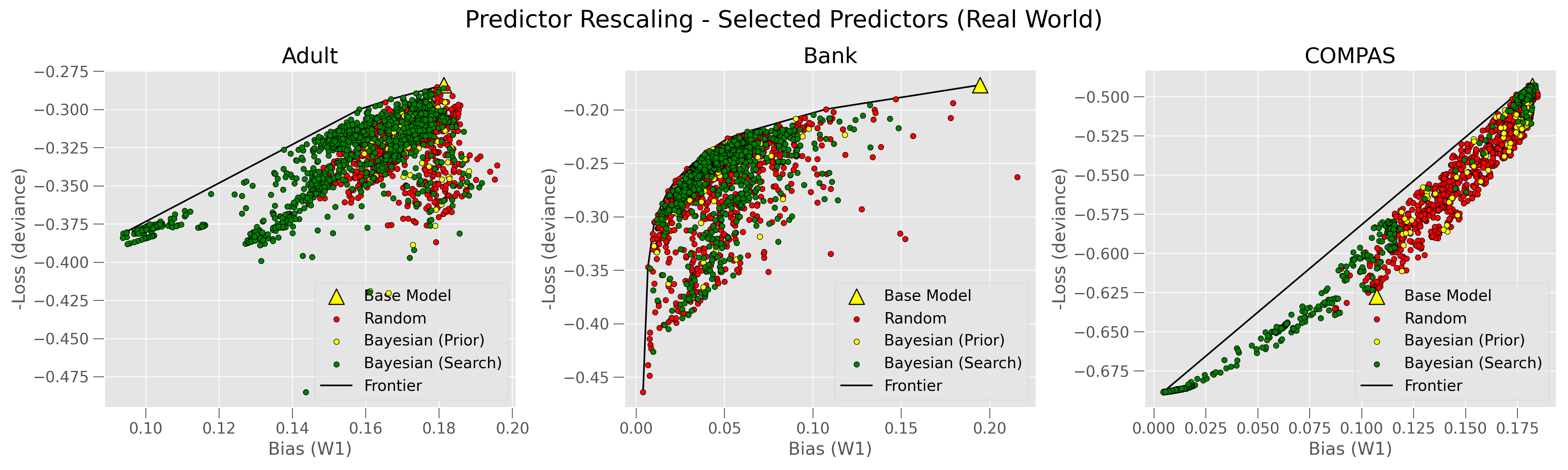}
\centering
\includegraphics[width=\textwidth]{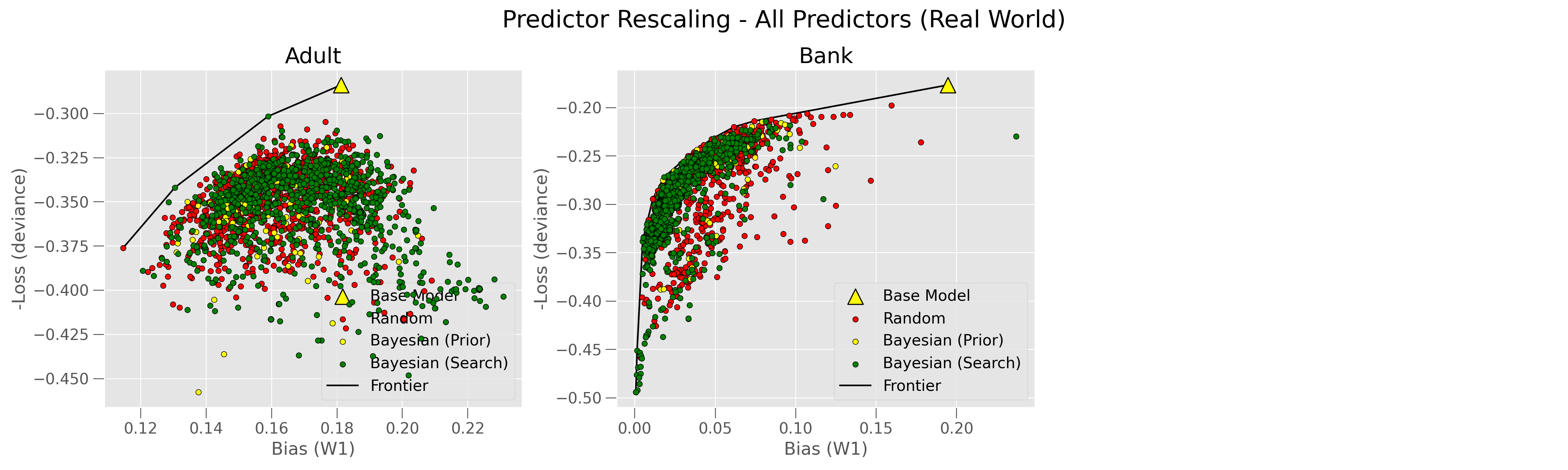}
\caption{\footnotesize  All models evaluated during Bayesian search and random search for the UCI Adult, UCI Bank Marketing, and COMPAS datasets. The first row displays the results for the predictor rescaling experiments presented earlier using selected predictors (for COMPAS, all predictors were selected). The second row displays results for analogous predictor rescaling experiments for UCI Adult and UCI Bank Marketing using all predictors. All results are presented on their respective test datasets.}\label{fig::BayesRealworld}
\end{figure}

\begin{figure}[ht!]
\centering
\includegraphics[width=\textwidth]{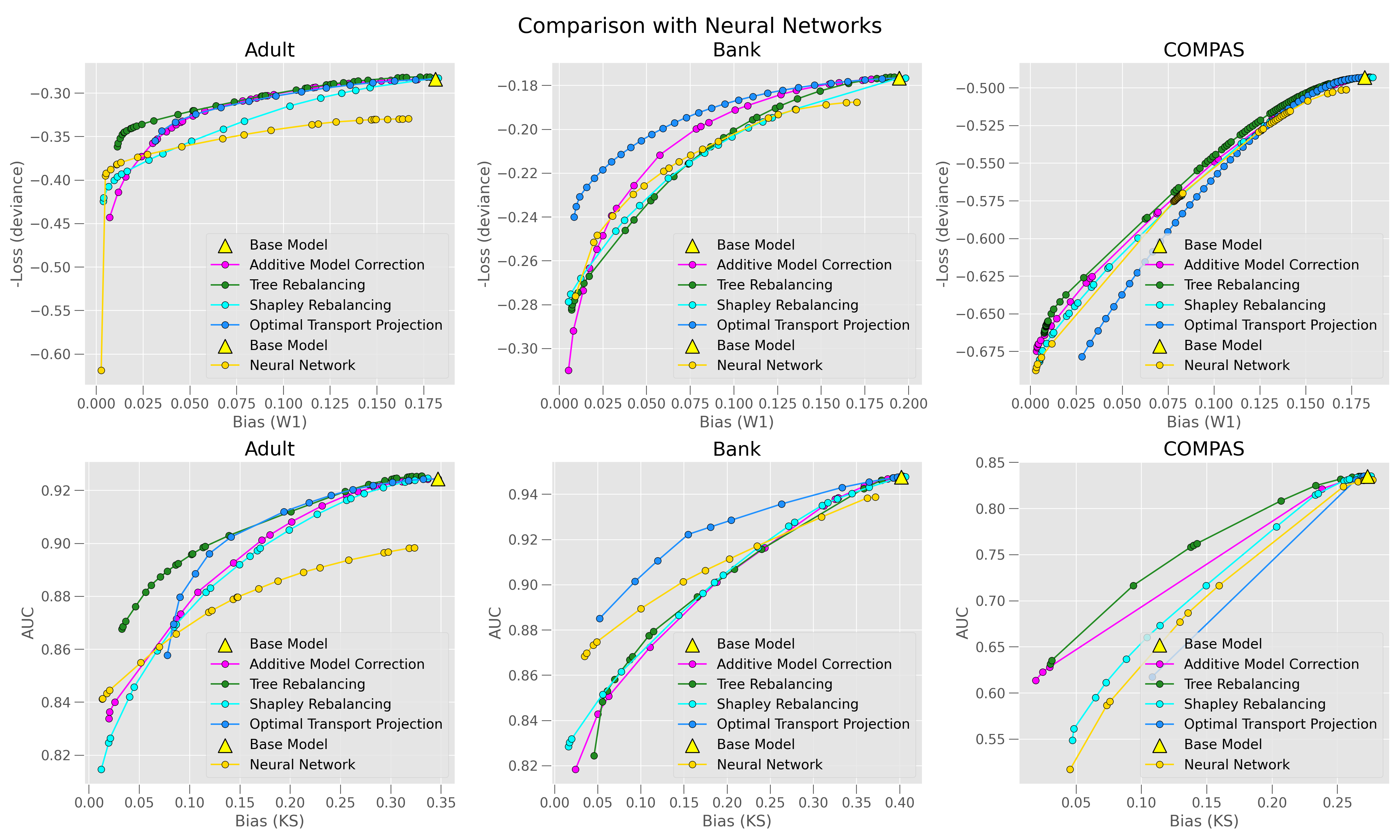}
\caption{Efficient frontiers for UCI Adult, Bank Marketing, and COMPAS datasets, evaluated on their test datasets.}
\label{fig::nnrealworld}
\end{figure}

\subsection{Bias mitigation using neural networks}
\label{app:nncomparison}
Many works have proposed to learn fairer models using gradient descent with a bias penalized loss function. For example, \cite{Jiang2020Jul} trained logistic models using a Wasserstein-based bias penalty. Similarly, \cite{Vogel2021} trained neural networks using a ROC-based bias penalty. While this work focuses on the application of gradient descent to explainable post-processing, we may employ similar procedures to train neural networks from scratch. Doing so allows optimization over larger families of functions but may pose challenges for explainability.

Figure \ref{fig::nnrealworld} compares the bias performance frontiers achieved by our various post-processing methods with the bias performance frontier achieved by training neural networks. Neural networks were trained using zero, one, two,  and three hidden layers with widths equal to the number of predictors in their corresponding training datasets. Depth was capped at three based on the observation that networks of depth three underperformed networks of depth two on UCI Adult and UCI Bank Marketing.

These results reveal some advantages and disadvantages of neural networks. On UCI Adult, neural networks are at a disadvantage relative to post-processing methods because the best performing neural network is far from the performance of the trained (CatBoost) base model used for post-processing. Thus, while one can reduce the bias of neural networks without large reductions in neural network performance, this does not fully compensate for the performance advantage of the CatBoost model.


In contrast, neural networks are better positioned on the UCI Bank Marketing dataset, achieving a best performance value lower than, but much closer to, that of the trained (CatBoost) base model. As a result, the neural network frontier bests several others (tree and Shapley rebalancing) outside the high performance regime. Furthermore, perhaps because linear models are more prone to score compression than non-linear models, neural networks beat all gradient descent based post-processing methods on the distribution invariant metric pair (AUC vs. KS) in the low bias regime.

Finally, on COMPAS, neural networks perform similarly to post-processing methods. When the number of features is small and the number of observations is limited, model complexity is not an important factor and most approaches may achieve similar results.

For more context on methodology, bias mitigation for neural networks was conducted using the Karush-Kuhn-Tucker approach
\citep{Karush,KuhnTucker} in a manner analogous to that used by our post-processing methodologies (using an adaptation of Algorithm \ref{GDOalgo} for parameters of non-linear models). As in \ref{app::perturbation}, we use minimization objectives of the form 
\[
(1-\omega_j)\cL + \omega_j\B,
\]
where $\cL$ is the binary cross-entropy and $\B$ is a modified version of \eqref{int_approx_est}, with $r_s(z)=\sigma(20z)$, $h(z)=z^2$, $\rho(t)=1$, $\Delta t=1/129$. We set $n_{perf}=n_{bias} = 1024$, learning rate $\alpha=0.01$, and $n_{epochs}=20$. Finally, we define $\omega_j = (C/20)\cdot j$ for $j\in \{0, 1,\dots, 20\}$, where $C$ is a suitable scaling constant.



\section{On optimal transport}\label{sec::kpminimization}

To formulate the transport problem we need to introduce the following notation. Let $\mathcal{B}(\RR^k)$ denote the $\sigma$-algebra of Borel sets. The space of all Borel probability measures on $\RR^k$ is denoted by $\mathscr{P}(\RR^k)$. The space of probability measure with finite $q$-th moment is denoted by 
\begin{equation*}\label{Borel_prob_meas}
\mathscr{P}_q(\RR^k)=\{ \mu\in \mathscr{P}(\RR^k): \int_{\RR^k} |x|^q d\mu(x) < \infty \}.
\end{equation*}

\begin{definition}[{\bf push-forward}]\label{def::push-forward}
\rm
\begin{itemize}

  \item []
  \item [(a)] Let $\PP$ be a probability measure on a measurable space $(\Omega,\mathcal{F})$. Let $X \in \RR^p$ be a random vector defined on $\Omega$. The push-forward probability distribution of $\PP$ by $X$ is defined by
  \[
  P_X(A):=\PP\big( \{ \omega \in \Omega: X(\omega) \in A \} \big).
  \]

  \item [(b)] Let $\mu \in \mathscr{P}(\RR^k)$ and $T:\RR^k \to \RR^m$ be Borel measurable, the pushforward of $\mu$ by $T$, which we denote by $T_{\#}\mu$ is the measure that satisfies
  \[
  (T_{\#}\mu)(B)=\mu\big(T^{-1}(B)\big), \quad B \subset \mathcal{B}(\RR^k).
  \]  

  \item [(c)] Given measure $\mu=\mu(dx_1,dx_2,...,dx_k) \in \mathscr{P}(\RR^k)$ we denote its marginals onto the direction $x_j$ by $(\pi_{x_j})_{\#}\mu$ and the cumulative distribution function by 
  \[
  F_{\mu}(a_1,a_2,\dots,a_k)=\mu( (-\infty,a_1]\times(-\infty,a_2] \dots, (-\infty,a_k])
  \]
\end{itemize}
\end{definition}

\begin{theorem}[\bf change of variable]\label{prop::changeofvar} Let $T:\RR^k \to \RR^m$ be Borel measurable map and $\mu \in \mathscr{P}(\RR)$. Let $g\in L^1(\RR^m, T_{\#}\mu)$. Then
\[
\int_{\RR^m} g(y) T_{\#}\mu(dy) = \int_{\RR^k} g(T(x)) \, \mu(dx).
\]
\end{theorem}
\begin{proof}
See \citet[p.~196]{Shiryaev}.
\end{proof}

\begin{proposition}\label{prop::push_mu} Let $\mu \in \mathscr{P}(\RR)$ and $F_{\mu}^{[-1]}$ be the pseudo-inverse of its cumulative distribution function $F_{\mu}$, then $\mu=(F_{\mu}^{[-1]})_{\#} \lambda|_{[0,1]}$, where $\lambda$ is the Lebesgue probability measure on $\RR$.
\end{proposition}
\begin{proof}
See \citet[p.~60]{Santambrogio2015}.
\end{proof}

\begin{definition}[\bf Kantorovich problem on $\RR$]\label{def::transport-cost}
 Let $\mu_1,\mu_2 \in \mathscr{P}(\RR)$ and $c(x_1,x_2) \geq 0$ be a cost function.  Consider the problem 
 \[
   \inf_{\gamma \in \Pi(\mu_1,\mu_2) } \Bigg\{ \int_{\RR^2} c(x_1,x_2) \gamma(dx_1,dx_2) \bigg\} =: \mathscr{T}_c(\mu_1,\mu_2)
\]
where  $\Pi(\mu_1,\mu_2)=\{ \gamma \in \mathscr{P}(\RR^2): (\pi_{x_j})_{\#}\gamma = \mu_j \}$
denotes the set of transport plans between $\mu_1$ and $\mu_2$, and $\mathscr{T}_c(\mu_1,\mu_2)$ denotes the minimal cost of transporting $\mu_1$ into $\mu_2$.
\end{definition}

\begin{definition} Let $q \geq 1$ and let $d(\cdot,\cdot)$ be a metric on $\RR^n$. Let the set 
\[\mathscr{P}_q(\RR^n;d)=\bigg\{\mu \in \mathscr{P}(\RR^n): \int d(x,x_0)^q d\mu(x) < \infty \bigg\}
\]
where $x_0$ is any fixed point. The Wasserstein distance $W_q$ on $\mathscr{P}_q(\RR^n;d)$ is defined by
\[
\begin{aligned}
  W_q(\mu_1,\mu_2;d) := \mathscr{T}^{1/q}_{d(x_1,x_2)^q}(\mu_1,\mu_2), \quad \mu_1,\mu_2 \in \mathscr{P}_q(\RR^n;d)
\end{aligned}
\]
where
\[
\mathscr{T}_{d(x_1,x_2)^q}(\mu_1,\mu_2) = \inf_{\gamma \in \mathscr{P}(\RR^2)} \bigg\{ \int_{\RR^2} d(x_1,x_2)^q d\gamma, \quad \gamma \in \Pi(\mu_1,\mu_2) \bigg\}.
\]
We drop the dependence on $d$ in the notation of the Wasserstein metric when $d(x,y)=|x-y|$.
\end{definition}

The following theorem contains well-known facts established in the texts such as \citet{Shorack1986,Villani2003,Santambrogio2015}.
\begin{theorem}\label{thm::transportprop} Let $\mu_1,\mu_2 \in \mathscr{P}(\RR)$. Let $c(x_1,x_2)=h(x-y) \geq 0$ with $h$ convex and let
\begin{equation*}
\pi^* := (F^{-1}_{\mu_1},F^{-1}_{\mu_2})_{\#} \lambda|_{[0,1]} \in \mathscr{P}(\RR^2)
\end{equation*}
where $\lambda|_{[0,1]}$ denotes the Lebesgue measure restricted to $[0,1]$. Suppose that $\mathscr{T}_c(\mu_1,\mu_2)<\infty$. Then

\begin{itemize}
  \item [(1)] $\pi^* \in \Pi(\mu_1,\mu_2)$ and $F_{\pi^*}=\min(F(a),F(b))$.

  \item [(2)] $\pi^*$ is an optimal transport plan that is
  \[
  \mathscr{T}_c(\mu_1,\mu_2)=\int_{\RR^2} h(x_1-x_2) \, d\pi^*(x_1,x_2).
  \]

  \item [(3)] $\pi^*$ is the only monotone transport plan, that is, it is the only plan that satisfies the property 
  \[
  (x_1,x_2),(x_1',x_2')\in {\rm supp(\pi^*)} \subset \RR^2\quad   x_1 < x_1' \quad \Rightarrow \quad x_2 \leq x_2'.
  \]

  \item [(4)] If $h$ is strictly convex then $\pi^*$ is the only optimal transport plan. 

  \item [(5)] If $\mu_1$ is atomless, then $\pi^*$ is determined by the monotone map $T^*=F_{\mu_2}^{[-1]}\circ F_{\mu_1}$, called an optimal transport map. Specifically, $\mu_2=T^*_{\#}\mu_1$ and hence $\pi^* = (I,T^*)_{\#}\mu_1$, where $I$ is the identity map. Consequently,
\[
\int_{\RR^2} h(x_1-x_2) \, d\pi^*(x_1,x_2) = \int_{\RR} h(x_1-T^*(x_1)) d\mu_1(x_1) = \E[X_1-T^*(X_1)], \quad \mu_1=P_{X_1}.
\]

\item [(6)] For $q \in [1,\infty)$, we have
\[
\begin{aligned}
{W_q}^q(\mu_1,\mu_2) &= \mathscr{T}_{|x_1-x_2|^q}(\mu_1,\mu_2) = \int_{\RR^2} |x_1-x_2|^q d\pi^*(x_1,x_2) \\
&= \int_0^1 |F^{[-1]}_{\mu_1}(p)-F^{[-1]}_{\mu_2}(p)|^q dp < \infty.
\end{aligned}
\]

\end{itemize}
\end{theorem}

\begin{definition}
  Given a set of probability measures $\{ \mu_j \}_{j=1}^J \subset \mathscr{P}_2(\RR^n)$, with $J \geq 1$, with finite second moments, and weights $\{ \omega_j \}_{j=1}^J$, the 2-Wasserstein barycenter is the minimizer of the map $\nu \to \sum_{j\in J} \omega_j W_2^2(\nu, \mu_j).$
\end{definition}

\section{Model bias with multiple protected attributes}\label{sec::multattr}

In this section, following the formulation in \cite{Miroshnikov2020}, we generalize the notion of model bias to the case of protected attributes with multiple classes. To this end, let $f$ be a model, $X=(X_1,\dots,X_n)$ the predictors, and $G \in \{0,\dots,K-1\}$ a protected attribute, with $G=0$ denoting the non-protected (majority) class. Suppose that $\Bias(f|X, G \in \{k,m\})$ is a model bias metric that estimates the distance (or divergence) between any two subpopulation distributions $P_{f(X)|G=k}$ and $P_{f(X)|G=m}$. In that case,  the total weighted model bias is defined to be 
\[
\Bias^{(w)}(f|X,G) = \sum_{k=1}^{K-1} w_k \Bias(f|X,G \in \{0,k\}),
\] 
where $w_k \geq 0$ are nonnegative weights for $k \in \{1,\dots,K-1\}$. For instance, the  weights can be chosen as $w_k=\PP(G=k)$.

Suppose now that subpopulation distributions $P_{f(X)|G=k}$, $k \in \{0,\dots,K-1\}$, have densities and that the pairwise model bias metric is defined to be the $q$-Wasserstein distance, with $q \geq 1$. In that case, the total model bias can be defined using a weighted sum of pairwise distances (raised to the $q$-th power) between each subpopulation distribution  $P_{f(X)|G=k}$ and their $q$-Wasserstein barycenter $\mu^{(q)}_{B}$:
\[
\Bias_{W^q_q}(f|X,G) = \sum_{k=0}^{K-1} p_k W_q^q(P_{f(X)|G=k},\mu^{(q)}_{B}),
\]
where $p_k=\PP(G=k)$, and where the barycenter \cite{AguehCarlier2011,Brizzi2025} is defined by
\begin{equation}\label{def::q-barycenter}
\mu^{(q)}_{B}=\argmin_{\nu \in \mathscr{P}_q(\RR^n)} \sum_{k=0}^{K-1}p_k W^q_q(\nu,\mu_B^{(q)}).
\end{equation}

It is worth noting that for $q>1$, there exists a unique minimizer of the right-hand side of \eqref{def::q-barycenter}, while for $q=1$, uniqueness fails. For this reason, $\mu_B^{(1)}$ can be obtained as the limit of $\mu_B^{(q)}$ as $q \to 1^+$; for details, see \cite[Theorem 5.1]{Brizzi2025}.

\section{Relaxation of distributions}

\begin{definition}\label{relax_family}
Let $Z$ be a random variable and $F_Z$ be its CDF. Suppose $\{r_s(t)\}_{s \in \RR_+}$ is a family of continuous functions such that the map $z \mapsto r_s(z)$ is non-decreasing and globally Lipschitz, and satisfies $r_s(z) \to 0$ as $z \to -\infty$ and  $r_s(z)\to 1$ as $z \to \infty$. The family of relaxed distributions associated with $\{r_s\}_{s \in \RR_+}$ is then defined by
\[
F_Z^{(s)}(t):=1-\E[r_s(Z-t)], \quad s \in \RR_+.
\]
\end{definition}

In what follows, we let $H(z):=\1_{\{z>0\}}$ be the left-continuous version of the Heaviside function.

\begin{lemma}\label{lmm::relax_stat}
Let $Z$ be a random variable. Let $r_s$ and $F_Z^{(s)}$, $s\in \RR_+$, be as in Definition \ref{relax_family}.

\begin{itemize}

\item [$(i)$] For $s>0$, $F_Z^{(s)}$ is a CDF, which is Lipschitz continuous on $\RR$, with $\Lip (F^{(s)}_Z) \leq \Lip(r_s)$. Furthermore, its pointwise derivative, which exists $\lambda$-a.s., is given by
\begin{equation*}
\frac{d}{dt} F^{(s)}_k(t) = \E\bigg[\frac{dr_s}{dz}(Z-t)\bigg] \geq 0, \quad \text{$\lambda$-a.s.}
\end{equation*}

\item [$(ii)$] If $\lim_{s \to \infty} r_s(z)=H(z)$ for all $z \in \RR$, then
\begin{equation}\label{tot_conv}
\lim_{s \to \infty} F_Z^{(s)}(t) = F_Z(t), \quad \forall t \in \RR.
\end{equation}

\item [$(iii)$] If $\lim_{s \to \infty} r_s(z)=H(z)$ for all $z \in \RR \setminus \{0\}$ and $\lim_{s \to \infty} r_s(0)=r_0>0=H(0)$, then
\[
\lim_{s \to \infty} F_Z^{(s)}(t) = F_Z(t) - r_0 \cdot \PP(Z=t), \quad \forall t \in \RR,
\]
in which case, the limit \eqref{tot_conv} holds if and only if $t \in \RR$ is a point of continuity of $F_Z$.

\end{itemize}

\begin{proof}
The statement $(i)$ follows directly from the definition of $r_s$, Lipschitz continuity, and the dominated convergence theorem. 

Suppose $\lim_{s \to \infty} r_s(z)=H(z)$ for all $z \in \RR$. Then 
\[
\lim_{s \to \infty} r_s(Z(\omega)-t) = H(Z(\omega)-t) = \1_{\{Z(\omega) > t\}}, \quad \forall \omega \in \Omega, \quad \forall t \in \RR,
\]
and hence, since $0 \leq r_s \leq 1$, by the dominated convergence theorem \cite{Royden2010}, we obtain
\begin{equation*}\label{relax_conv}
\lim_{s \to \infty} \E[r_s(Z-t)] = \P(Z>t) = 1 - F_Z(t), \quad \forall t \in \RR.
\end{equation*}
This proves $(ii)$.

Suppose $r_s(0) \to r_0>0$ as $s \to \infty$. Let $\Omega_{t_0}=\{\omega: Z(\omega)=t_0\}$. Then for any $\omega \in \Omega \setminus \Omega_t$, we must have $\lim_{s \to \infty} r_s(Z(\omega)-t_0)=\1_{\{Z(\omega) > t_0\}}$. Then, by the dominated convergence theorem, we have
\begin{equation*}\label{spec_relax}
\begin{aligned}
\lim_{s \to \infty} \E[r_s(Z-t_0)] &=\lim_{s \to \infty} \E[r_s(Z-t_0) \1_{\Omega_{t_0}} ] + \lim_{s \to \infty} \E[r_s(Z-t_0) \1_{(\Omega \setminus\Omega_{t_0})}]  \\
& =\lim_{s \to \infty} \E[r_s(0) \1_{\Omega_{t_0}} ] + \lim_{s \to \infty} \E[r_s(Z-t) \1_{(\Omega \setminus\Omega_{t_0})}] \\
& = r_0 \cdot \PP(\Omega_{t_0}) + \E[\1_{\{Z>t_0\}} \1_{(\Omega \setminus \Omega_{t_0})}] = r_0 \cdot \PP(\Omega_{t_0}) + 1-F_Z(t_0).
\end{aligned}
\end{equation*}

Thus, $\lim_{s \to \infty} F_Z^{(s)}(t_0) = F_Z(t_0) - r_0 \PP(\Omega_{t_0})$. Given that $r_0>0$, we conclude that \eqref{tot_conv} holds at $t=t_0$ if and only if $\PP(\Omega_{t_0})=0$, which holds if and only if $t=t_0$ is a point of continuity of $F_Z$. This proves $(iii)$.
\end{proof}

\end{lemma}

\begin{lemma}
Let $Z_0,Z_1$ be random variables with $F_{Z_0},F_{Z_1}$ denoting their CDFs. Let $r_s$ and $F_{Z_k}^{(s)}$, $k \in \{0,1\}$, $s\in \RR_+$, be as in Definition \ref{relax_family}. Suppose $\mu$ is a Borel probability measure, and $c(\cdot,\cdot)$ is continuous on $[0,1]^2$.

\begin{itemize}

\item [$(i)$] If $\lim_{s \to \infty} r_s(z)=H(z)$ for all $z \in \RR$, then
\begin{equation}\label{bias_relax_lim}
\int c(F_{Z_0}(t),F_{Z_1}(t)) \, \mu(dt) = \lim_{s \to \infty} \int c \big(F_{Z_0}^{(s)}(t),F_{Z_1}^{(s)}(t) \big) \, \mu(dt).
\end{equation}

\item [$(ii)$] Suppose $\lim_{s \to \infty} r_s(z)=H(z)$ for all $z \in \RR \setminus \{0\}$ and $\lim_{s \to \infty} r_s(0)=r_0>0=H(0)$. Then \eqref{bias_relax_lim} holds if $\mu$ has the following property: $\mu(\{z_*\})=0$ whenever $z_* \in A_0 \cup A_1$, where $A_0$ and $A_1$ are sets containing atoms of $P_{Z_0}$ and $P_{Z_1}$, respectively.
\end{itemize}
\end{lemma}

\begin{proof}
Suppose $\lim_{s \to \infty} r_s(z)=H(z)$ for all $z \in \RR$. Since $c$ is continuous on $[0,1]^2$, by Lemma \ref{lmm::relax_stat}$(ii)$, we have
\[
\lim_{s \to \infty}c(F^{(s)}_{Z_0}(t),F^{(s)}_{Z_1}(t)) = c(F_{Z_0}(t),F_{Z_1}(t))\]
for all $t \in \RR$. Then, since $c$ is bounded on $[0,1]^2$, by the dominated convergence theorem, we obtain \eqref{bias_relax_lim}. This gives $(i)$. 

If $\mu(\{z_*\})=0$ whenever $z \in A_0 \cup A_1$, then $\mu(A_0 \cup A_1)=0$, as $A_0$ and $A_1$ are at most countable. Hence by Lemma \ref{lmm::relax_stat}$(iii)$, we get
$\lim_{s \to \infty}c(F^{(s)}_{Z_0}(t),F^{(s)}_{Z_1}(t)) = c(F_{Z_0}(t),F_{Z_1}(t))$ $\mu$-almost surely. Hence, using  the dominated convergence theorem again, we obtain \eqref{bias_relax_lim}. This establishes $(ii)$.

\end{proof}

\subsection {\bf Proof of Theorem \ref{thm::MCestimator}}

\begin{proof}\label{proof::MCestimator}
In what follows we suppress the dependence on $\theta$, assuming it is fixed. 

Fix $s>0$. To show that 
\[
\displaystyle \hat{I}_{m_0,m_1,T}^{(s)}:=\frac{1}{T} \sum_{j=1}^T h(\hat{\bf B}_{s,m_0,m_1}({\bf t}^{(j)}))
\]
converges in the mean squared sense to 
\[
\int h(B_s(t)) \, \mu(dt)=\E_{t \sim \mu} [ h(B_s(t))],\]
we require
  \begin{equation*}
   \lim_{m_0,m_1,T\to \infty}\E[(\mathcal{E}^{(s)}_{m_0,m_1,T})^2]=0, \quad \text{where}  \quad \mathcal{E}^{(s)}_{m_0,m_1,T}:=\hat{I}^{(s)}_{m_0,m_1,T} - \E_{t \sim \mu} [ h(B_s(t))].
  \end{equation*}

To this end, define $\displaystyle I_{T}^{(s)}:=\frac{1}{T} \sum_{j=1}^T h(B_s({\bf t}^{(j)}))$. We then have
  \begin{align*}
   \E[(\mathcal{E}^{(s)}_{m_0,m_1,T})^2] &= \E[(\hat{I}^{(s)}_{m_0,m_1,T} - \E_{t \sim \mu} [ h(B_s(t))])^2]\\
   &\le 2\left(\E[(\hat{I}^{(s)}_{m_0,m_1,T} - I_{T}^{(s)})^2] + \E[(I_{T}^{(s)} - \E_{t \sim \mu} [ h(B_s(t))])^2]\right)\\
   &= 2\left(\E[(\hat{I}^{(s)}_{m_0,m_1,T} - I_{T}^{(s)})^2] + \Var[I^{(s)}_{T}]\right),
  \end{align*}
  where the last line is a consequence of $\E[I^{(s)}_{T}]=\frac{1}{T}\sum_{j=1}^T \E_{t^{(j)}\sim \mu}[B_s(t^{(j)})]=\E_{t \sim \mu} [ h(B_s(t))]$.

Next, note that by assumption, there is a Lipschitz constant $C_h$ such that $|h(a)-h(b)|\leq C_h |a-b|$ for $a,b \in [-1,1]$. Since $\hat{\bf B}_{s,m_0,m_1}(t), B_s(t) \in [-1,1]$,  the first term in the bound above satsfies:
\begin{equation}\label{term1-est} 
\begin{aligned}
 \E[(\hat{I}^{(s)}_{m_0,m_1,T} - I_{T}^{(s)})^2]&=\frac{1}{T^2} \E\left[\left(\sum_{j=1}^T \left(h(\hat{\bf B}_{s,m_0,m_1}({\bf t}^{(j)}))-h(B_s({\bf t}^{(j)}))\right)\right)^2\right]\\
 &\le \frac{T}{T^2}  \E\left[ \sum_{j=1}^T \left(h(\hat{\bf B}_{s,m_0,m_1}({\bf t}^{(j)}))-h(B_s({\bf t}^{(j)}))\right)^2\right]\\
 &= \frac{C_h^2}{T}  \sum_{j=1}^T \E\left[\left(\hat{\bf B}_{s,m_0,m_1}({\bf t}^{(j)})-B_s({\bf t}^{(j)})\right)^2\right].
\end{aligned}
\end{equation}

Observe that for fixed $t\in \RR$ we have
\begin{equation}\label{biasterm-est}
  \begin{aligned}
   &\E\left[\left(\hat{\bf B}_{s,m_0,m_1}(t)-B_s(t)\right)^2\right]\\
   & =  \E\left[\left(\frac{1}{m_1} \sum_{i=1}^{m_1} r_s(f({\bf x}^{(i)}_1)-t) - \frac{1}{m_0} \sum_{i=1}^{m_0} r_s(f({\bf x}^{(i)}_0)-t)-F^{(s)}_0(t) + F^{(s)}_1(t)\right)^2\right]\\
   & \le  2\E\left[\left(\frac{1}{m_1} \sum_{i=1}^{m_1} r_s(f({\bf x}^{(i)}_1)-t) - (1-F^{(s)}_1(t))\right)^2\right] \\
   & \quad + 2\E\left[\left(\frac{1}{m_0} \sum_{i=1}^{m_0} r_s(f({\bf x}^{(i)}_0)-t)-(1-F^{(s)}_0(t))\right)^2\right]\\
   = & 2 \Var\left[\frac{1}{m_1} \sum_{i=1}^{m_1} r_s(f({\bf x}^{(i)}_1)-t)\right] + 2\Var\left[\frac{1}{m_0} \sum_{i=1}^{m_0} r_s(f({\bf x}^{(i)}_0)-t)\right]\\
   = & \frac{2}{m_1}\Var\left[r_s(f({\bf x}^{(1)}_1)-t)\right] + \frac{2}{m_0}\Var\left[r_s(f({\bf x}^{(1)}_0)-t)\right]\\
   \le & \frac{2}{m_1} + \frac{2}{m_0} \le \frac{4}{\min(m_0,m_1)},
\end{aligned}
\end{equation}
where we used the relations \eqref{relax_cdf}, \eqref{loc_bias_relax}, and the fact that ${\bf x}_k^{(i)}$, $i \in \{1,\dots,m_k\}$, are i.i.d.\ random variables.

Therefore, using the fact that each ${\bf t}^{(j)}$ is independent of ${\bf D}_0 \cup {\bf D}_1$, we obtain the following bound for the first term
  \begin{align*}
   \E[(\hat{I}_{m_0,m_1,T}^{(s)} - I_{T}^{(s)})^2]&= \frac{C_h^2}{T}  \sum_{j=1}^T \E_{t^{(j)}\sim \mu} \bigg[\E\left[\left(\hat{\bf B}_{s,m_0,m_1}(t^{(j)})-B_s(t^{(j)})\right)^2\right] \bigg]\\
   & \leq \frac{C_h^2}{T}  \sum_{j=1}^T \E_{t^{(j)}\sim \mu}\left[\frac{4}{\min(m_0,m_1)}\right]= \frac{4C_h^2}{\min(m_0,m_1)}.
  \end{align*}

Next, for the second term, we have
  \begin{equation*}
   \Var[I_{T}^{(s)}]=\frac{1}{T^2}\Var\left[\sum_{j=1}^T h(B_s({\bf t}^{(j)}))\right]=\frac{1}{T^2} \sum_{j=1}^T \Var\left[h(B_s({\bf t}^{(j)}))\right] \le \frac{1}{T^2} \left[\sum_{j=1}^T 2C_h\right]=\frac{2C_h}{T},
  \end{equation*}
  where we used the fact that ${\bf t}^{(j)}$, $j \in \{1,\dots,T\}$, are independent, and where the inequality follows from the observation that $\Var\left[h(B_s({\bf t}^{(j)}))\right]=\Var\left[h(B_s({\bf t}^{(j)}))-h(0)\right]$, and $h(B_s({\bf t}^{(j)}))-h(0) \in [-C_h, C_h]$.

Finally, putting the bounds of both terms together we obtain
  \begin{equation*}
   \E[(\mathcal{E}^{(s)}_{m_0,m_1,T})^2] \le \frac{8C_h^2}{\min(m_0,m_1)} + \frac{4C_h}{T}
  \end{equation*}
which converges to $0^+$ as $m_0,m_1,T\to \infty$. The bound also suggests that in order to obtain the optimal rate of convergence, $T$ and $\min(m_0,m_1)$ must scale at the same rate. Thus, if $T=c\cdot \min(m_0,m_1)$, the mean squared error scales optimally with rate $O(T^{-1})$, uniformly in $s$, or equivalently, the error scales as $O(T^{-1/2})$ in $L^2(\P)$. This completes the proof.
\end{proof}

\subsection{\bf Proof of Theorem \ref{thm::DiscrEstimator}} \label{proof::DiscrEstimator}

\begin{proof} Fix $s>0$. Then, suppressing the dependence on $\theta$, assuming it is fixed, set
 \[
   \hat{I}^{(s)}_{m_0,m_1,T}:=\frac{1}{T}\sum_{j=1}^T h(\hat{\bf B}_{s,m_0,m_1}(t_j)) \rho(t_j) \quad \text{and} \quad I^{(s)}_{T}:=\frac{1}{T}\sum_{j=1}^T h(B_s(t_j)) \rho(t_j),
\]
and define the error term by
\[
  \mathcal{E}^{(s)}_{m_0,m_1,T}:=\hat{I}^{(s)}_{m_0,m_1,T}-\int_0^1 h \big( B_s(t) \big) \mu(dt).
 \]

Then the mean squared error satisfies
  \begin{align*}
   \E[(\mathcal{E}^{(s)}_{m_0,m_1,T})^2] 
   \le 2\left(\E[(\hat{I}_{m_0,m_1,T}^{(s)} - I^{(s)}_{T})^2] + \Big(I_{T}^{(s)}- \int_0^1 h \big( B_s(t) \big) \mu(dt) \Big)^2 \right).
  \end{align*}

To estimate the first term in the bound for the mean squared error, we follow computations similar to \eqref{term1-est} and \eqref{biasterm-est} as in the proof of Theorem \ref{thm::MCestimator} and obtain
  \begin{equation*}
   \E[(\hat{I}_{m_0,m_1,T}^{(s)} - I_{T}^{(s)})^2] \le \frac{4C_h^2}{\min(m_0,m_1)}.
  \end{equation*}

We next investigate the second term in the above bound. Observe that we can write
  \begin{align*}
   I_{T}^{(s)} -\int_0^1 h \big( B_s(t) \big) \mu(dt) &=\frac{1}{T}\sum_{j=1}^T h(B_s(t_j)) \rho(t_j) - \int_0^1 h(B_s(t))\rho(t)dt\\
   &=\sum_{j=1}^T \int_{t_{j-1}}^{t_j}\Big(h(B_s(t_j)) \rho(t_j)-h(B_s(t))\rho(t)\Big)dt
  \end{align*}

Since $h$ and $\rho$ are Lipschitz on $[-1,1]$ and $[0,1]$, respectively, we set $C_h:=\Lip(h)$, $C_{\rho}:=\Lip(\rho)$, and $M_h:=\max_{a\in[-1,1]} h(a)$, and note that $\rho(t)\leq (C_{\rho}+1)$, $t \in [0,1]$. Furthermore, by assumption, $\Lip(r_s) \leq sC_r$, which together with \eqref{relax_cdf} and \eqref{loc_bias_relax} implies that $B_s$ is globally Lipschitz continuous, with $\Lip(B_s) \leq s C_r$.

Combining the above, we obtain that $(h\circ B_s) \cdot \rho$ is Lipschitz continuous on $[0,1]$ satisfying
\begin{equation*} 
\Lip( ((h\circ B_s) \cdot \rho) |_{[0,1]}) \leq  (C_{\rho}+1) \cdot (s \, C_r C_h + M_h)=:C(h,\rho,r,s).
\end{equation*}
Revisiting the second term, we then have:
\begin{equation}\label{int_approx_det_}
\begin{aligned}
   \Big|I_{T}^{(s)} -\int_0^1 h \big( B_s(t) \big) \mu(dt) \Big|&\le \sum_{j=1}^T \int_{t_{j-1}}^{t_j}C(h,\rho,r,s)(t_j-t)dt\\
   &= \frac{1}{2}C(h,\rho,r,s)\sum_{j=1}^T (t_j-t_{j-1})^2 
   = \frac{1}{2}C(h,\rho,r,s)\Delta t = O((1+s)\Delta t).
\end{aligned}
\end{equation}

Putting the two estimates together yields
  \begin{equation*}
   \E[(\mathcal{E}^{(s)}_{m_0,m_1,T})^2] \le \frac{8C_h^2}{\min(m_0,m_1)} + \frac{1}{2}C(h,\rho,r,s)^2(\Delta t)^2 = O\left((\min(m_0,m_1))^{-1}+(1+s)^2T^{-2}\right),
  \end{equation*}
which converges to $0^+$ as $m_0,m_1,T\to \infty$. The bound also suggests that in order to obtain the optimal rate of convergence, $(T/(1+s))^2$ and $\min(m_0,m_1)$ must scale at the same rate. Thus, if $T/(1+s)=c\cdot \sqrt{\min(m_0,m_1)}$, the mean squared error scales optimally with rate $O(((1+s)/T)^2)$, or equivalently, the error scales as $O((1+s)/T)$ in $L^2(\P)$.

\end{proof}

\section{Quantile transformed distributions with atoms}\label{app::auxlemmas}

Let $\mu \in \mathscr{P}(\RR)$ and $F_{\mu}$ denote its CDF. It is well-known that the generalized inverse $F_{\mu}^{[-1]}$ satisfies the Galois inequalities (see \cite{Santambrogio2015})
\begin{equation}\label{Galois_ineq}
t < F^{[-1]}(q)  \, \Leftrightarrow \, F_{\mu}(t)<q, \quad t\in\RR, \,\, q\in(0,1).
\end{equation}
Replacing the sign  $<$ with $\leq$, however, is in general not possible, unless $\mu$ is atomless and its support is connected (see Lemma \ref{lmm::quant_transf}). However,  adjusting $F_{\mu}$ and the generalized quantile function $F_{\mu}^{[-1]}$ appropriately  allows for the statement with $\leq$. To this end, we define the following.

\begin{remark} \rm
Here, we use the convention that, whenever $F$ is a CDF, $F(-\infty)=0$ and $F(+\infty)=1$.
\end{remark}

\begin{definition}\label{def::adjusted_cfd_inverse} \rm
Let $\mu \in \mathscr{P}(\RR)$, and $F_{\mu}$ and $F_{\mu}^{[-1]}$ be its CDF and generalized inverse function, respectively. Define $\tF_{\mu}(t):=F_{\mu}(t^-)=\lim_{\tau \to t^-} F_{\mu}(\tau)$, $t \in \RR$, to be the left-continuous realization of $F_{\mu}$. Similarly, define $\tF_{\mu}^{[-1]}(q):=F^{[-1]}(q^+)=\lim_{p \to q^+} F_{\mu}(p)$, $q \in [0,1)$, and $\tF_{\mu}^{[-1]}(1)=+\infty$, to be the right-continuous realization of $F_{\mu}^{[-1]}$ on $[0,1]$.
\end{definition}

\begin{lemma}\label{lmm::Galois_ineq_adj}
Let $\mu \in \mathscr{P}(\RR)$. Let $F_{\mu}$, $\tF_{\mu}$, $F^{[-1]}_{\mu}$, and $\tF_{\mu}^{[-1]}$ be as in Definition \ref{def::adjusted_cfd_inverse}. Then
\begin{equation*}\label{Galois_ineq_adj}
 t \leq \tF_{\mu}^{[-1]}(q) \,\, \Leftrightarrow \,\, \tF_{\mu}(t) \leq q, \quad t \in \RR, \,\, q \in (0,1).
\end{equation*}
\end{lemma}
\begin{proof}
Take any $q \in (0,1)$ and $t \in \RR$. First, suppose that $t \leq \tF_{\mu}^{[-1]}(q)=F_{\mu}^{[-1]}(q^+)$. Take any $\delta>0$ and any $\varepsilon>0$ such that $q+\varepsilon<1$. Then $ t-\delta<t \leq \tF_{\mu}^{[-1]}(q)=F_{\mu}^{[-1]}(q^+) \leq F_{\mu}^{[-1]}(q+\varepsilon)$.

Then by \eqref{Galois_ineq}, $F_{\mu}(t-\delta) < q + \epsilon$. Since $\delta>0$ and $\varepsilon>0$ are arbitrary, we obtain $\tF_{\mu}(t)=F_{\mu}(t^-) \leq q$.

Suppose now that $F_{\mu}(t^-)=\tF_{\mu}(t) \leq q$. Then for any $\delta>0$ and any sufficiently small $\varepsilon>0$ such that $q+\epsilon < 1$, we have $F_{\mu}(t-\delta) \leq F_{\mu}(t^-)=\tF_{\mu}(t) \leq q < q+ \epsilon$. Then by \eqref{Galois_ineq}, $t -\delta < F_{\mu}^{[-1]}(q +\epsilon)$. Since $\delta>0$ and $\varepsilon>0$ are arbitrary, we conclude that $ t \leq F_{\mu}^{[-1]}(q^+)=\tF_{\mu}^{[-1]}(q)$.
\end{proof}

\begin{lemma}\label{lmm::quant_transf}
Let $Z$ be a random variable and $F_Z$ be its CDF. Let $\mu \in \mathscr{P}(\RR)$. Let $F_{\mu}$, $\tF_{\mu}$, $F^{[-1]}_{\mu}$, and $\tF_{\mu}^{[-1]}$ be as in Definition \ref{def::adjusted_cfd_inverse}.  Then $F_{\tF_{\mu}(Z)}$, the CDF of $\tF_{\mu}(Z)$, satisfies for any $a \in \RR$
\begin{equation}\label{cdf_quant_transf_mu_gen}
F_{\tF_{\mu}(Z)}(a) = \PP(\tF_{\mu}(Z)\leq a)= \left \{ 
\begin{aligned}
&= \PP(Z \leq \tF_{\mu}^{[-1]}(a)) = F_{Z} \circ \tF_{\mu}^{[-1]}(a), \quad a \in (0,1).\\
& = F_{Z} \circ \tF_{\mu}^{[-1]}(0^+), \quad a=0.\\
& = \1_{\{a \geq 1\}}, \quad a \in \RR \setminus [0,1)\\
\end{aligned}
\right.
\end{equation}

Hence, if $\mu$ is atomless, then $F_{F_{\mu}(Z)}$, the CDF of $F_{\mu}(Z)$, satisfies
\begin{equation}\label{cdf_quant_transf_mu_atomless}
F_{F_{\mu}(Z)}(q)=\PP(F_{\mu}(Z)\leq q) = F_{Z} \circ \tF_{\mu}^{[-1]}(q), \quad q \in (0,1).
\end{equation}
\end{lemma}

\begin{proof}
Pick any $a \in (0,1)$. By Lemma \ref{lmm::Galois_ineq_adj}, 
\[\{\omega \in \Omega: \tF_{\mu}(Z(\omega)) \leq q \} = \{\omega \in \Omega: Z(\omega) \leq \tF_{\mu}^{[-1]}(q) \}\]
and hence  \eqref{cdf_quant_transf_mu_gen}$_1$ holds. The rest of the statement follow from the right-continuity of $F_{\tF_{\mu}}$. If $\mu$ is atomless, then $F_{\mu}=\tF_{\mu}$ on $\RR$. Hence \eqref{cdf_quant_transf_mu_gen}$_1$ implies \eqref{cdf_quant_transf_mu_atomless}.
\end{proof}

\begin{remark}\rm
In Lemma \ref{lmm::quant_transf}, if $\mu$ is atomless, then $F_{\mu}=\tF_{\mu}$, but $F_{\mu}^{[-1]}$ is not necessarily equal to $\tF_{\mu}^{[-1]}$ unless the support of $\mu$ is connected, in which case, $F_{F_{\mu}(Z)}(a)=F_{Z} \circ F_{\mu}^{[-1]}(a)$, $a \in (0,1)$.
\end{remark}

\begin{remark}\label{rem::quant_trasf}\rm
If the sets containing the atoms of $\mu$ and the atoms of $P_Z$ have a nonempty intersection, then $F_{\mu}(Z) \neq \tF_{\mu}(Z)$ $\PP$-a.s. Hence, by the right-continuity of CDFs, there exists an open interval $I \subseteq [0,1]$ where $F_{F_{\mu}(Z)}$ differs from $F_{\tF_{\mu}(Z)}$. Then, by Lemma \ref{lmm::quant_transf}, $F_{F_{\mu}(Z)}$ differs from $F_{Z} \circ F_{\mu}^{[-1]}$ $\lambda$-a.s. on $I$. 
\end{remark}

\begin{proposition}\label{prop::quant_transf_bias}
Let $Z_0,Z_1$ be random variable, with $F_{Z_0},F_{Z_1}$ denoting their CDFs. Let $\mu \in \mathscr{P}(\RR)$, and $F_{\mu}$, $\tF_{\mu}$, $F^{[-1]}_{\mu}$, and $\tF_{\mu}^{[-1]}$ be as in Definition \ref{def::adjusted_cfd_inverse}. Suppose that $c(\cdot,\cdot)$ is continuous on $[0,1]^2$. Then

\begin{equation}\label{quant-transf-bias}
\int c(F_{Z_0}(t),F_{Z_1}(t)) \mu(dt) = \int_0^1 c(F_{\tF_{\mu}(Z_0)}(q),F_{\tF_{\mu}(Z_1)}(q)) \, dq.
\end{equation}
Hence, if $\mu$ is atomless, then $\tF_{\mu}$ can be replaced with $F_{\mu}$ in \eqref{quant-transf-bias}.
\end{proposition}

\begin{proof}
 By Proposition \ref{prop::push_mu}, $\mu = (F_{\mu}^{[-1]})_{\#} \lambda|_{[0,1]}$, and hence by Theorem \ref{prop::changeofvar} we obtain
\begin{equation}\label{push_mu}
\int c(F_{Z_0}(t),F_{Z_1}(t)) \, \mu(dt) = \int_0^1 c(F_{Z_0}\circ F_{\mu}^{[-1]}(q),F_{Z_1}\circ F_{\mu}^{[-1]}(q)) \, dq.
\end{equation}
By Lemma \ref{lmm::quant_transf},  $F_{\tF_{\mu}(Z_k)}(q)=F_{Z_k} \circ \tF_{\mu}^{[-1]}(q)$, for $q \in (0,1)$. Since the number of points where $F_{\mu}^{[-1]}$ has jumps is at most countable, we must have $\tF_{\mu}^{[-1]}=F_{\mu}^{[-1]}$ $\lambda$-a.s. on $[0,1]$. Hence $F_{\tF_{\mu}(Z_k)}=F_{Z_k} \circ F_{\mu}^{[-1]}$ $\lambda$-a.s., and hence, using \eqref{push_mu}, we obtain \eqref{quant-transf-bias}. If $\mu$ is atomless, then $F_{\mu}=\tF_{\mu}$ on $\RR$, and hence we have
\begin{equation}\label{quant-transf-bias-atomless}
\int c(F_{Z_0}(t),F_{Z_1}(t)) \mu(dt) = \int_0^1 c(F_{F_{\mu}(Z_0)}(q),F_{F_{\mu}(Z_1)}(q)) \, dq.
\end{equation}
\end{proof}

\begin{example}\label{ex::disc_score_inv}\rm
When $\mu$ has atoms, \eqref{quant-transf-bias} in general does not equal to \eqref{quant-transf-bias-atomless} in light of Remark \ref{rem::quant_trasf}.  The following measures provide a counter example: $P_{Z_0}=\delta_0$, $P_{Z_1}= \1_{[0,1]}(z) dz$, and $\mu=\frac{1}{2}(P_{Z_0}+P_{Z_1})$.
\end{example}

\citet{Becker2024} established that for continuous random variables $Z_0$, $Z_1$, and $Z$
\begin{equation*}\label{quant-transf-bias-W1-atomless}
\int |F_{Z_0}(t)-F_{Z_1}(t)| \, p_Z(t) dt = \int |F_{F_{Z}(Z_0)}(q)-F_{F_{Z}(Z_1)}(q) |\, dq = W_1(P_{F_{Z}(Z_0)},P_{F_{Z}(Z_1)}).
\end{equation*}
However, when $Z$ has atoms, in light of Example \ref{ex::disc_score_inv}, the above is in general not true. Below is a more general version of the above statement that follows directly from Proposition \ref{prop::quant_transf_bias}:
\begin{corollary}\label{corr::abs-dist-cost}
Let $Z_0,Z_1$ be random variable, with $F_{Z_0},F_{Z_1}$ denoting their CDFs. Let $\mu \in \mathscr{P}(\RR)$, and $F_{\mu}$, $\tF_{\mu}$, $F^{[-1]}_{\mu}$, and $\tF_{\mu}^{[-1]}$ be as in Definition \ref{def::adjusted_cfd_inverse}. Then
\begin{equation}\label{quant-transf-bias-W1}
\int |F_{Z_0}(t)-F_{Z_1}(t)| \mu(dt) = \int |F_{\tF_{\mu}(Z_0)}(q)-F_{\tF_{\mu}(Z_1)}(q) |\, dq = W_1(P_{\tF_{\mu}(Z_0)},P_{\tF_{\mu}(Z_1)}).
\end{equation}
If $\mu$ is atomless, then the right-hand side of \eqref{quant-transf-bias-W1} equals $W_1(P_{F_{\mu}(Z_0)},P_{F_{\mu}(Z_1)})$.
\end{corollary}

Thus, \eqref{quant-transf-bias-W1} allows for the generalization of \citet[Theorem 3.4]{Becker2024} for the case when the classification scores have atoms; see Proposition \ref{prop:gen_invariant_bias} in the main text. We note that a similar adjustment as in \eqref{inv_bias_gen} is required for other types of bias such as Equal Opportunity (${\rm EO}$) and Predictive Equality (${\rm PE}$), as discussed in Theorem 3.4 of \cite{Becker2024}.

\vspace{10pt}

\noindent{\bf Proof of Proposition \ref{prop::transp-cost-form}}\label{app::prop::transp-cost-form}
\begin{proof} By Proposition \ref{prop::push_mu}, $\mu=(F^{[-1]}_{\mu})_{\#} \lambda|_{[0,1]}$ and hence by Theorem \ref{prop::changeofvar}, we obtain
\[
\int c(F_0(t),F_1(t))  \mu(dt) = \int_0^1 h(F_0 \circ F^{[-1]}_{\mu}(q) - F_1 \circ F^{[-1]}_{\mu}(q)) \, dq.
\]

By construction, $F^{[-1]}_{\mu}=F^{-1}_{\mu}$, $F^{[-1]}_{Z_k}=F^{-1}_{Z_k}$ are well-defined inverses of $F_{\mu}$ and $F_{Z_k}$ on $[0,1]$, respectively. Let $\cT \sim \mu$ and $A_k=F_k(\cT)$. Then, the support of $P_{A_k}$ is $[0,1]$ and $F_{A_k}(t)=F_{\mu} \circ F_k^{-1}(t)$ for $t \in [0,1]$. Furthermore, the inverse of $F_{A_k}$ is well-defined on $[0,1]$ and equals $F_{A_k}^{-1}=F_k \circ F_{\mu}^{-1}$. Hence by Theorem \ref{thm::transportprop} we obtain
\[
\int_0^1 h(F_0 \circ F^{[-1]}_{\mu}(q) - F_1 \circ F^{[-1]}_{\mu}(q)) \, dq = \int_0^1 h(F^{[-1]}_{A_0}(q)-F^{[-1]}_{A_1} (q)) \, dq = \mathscr{T}_c(P_{F_0(\cT)},P_{F_1(\cT)}).
\]
The result follows from the above relationship and the fact that $P_{F_k(\cT)}={F_k}_{\#}\mu$, $k \in \{0,1\}$.
\end{proof}

\begin{remark} \label{rem::nonunif-quantiles-bias}\rm
The distribution-invariant model bias \eqref{inv_bias}, assuming $P_{Z}$ has a density, can be expressed as follows \cite[Theorem 3.4]{Becker2024}:
\begin{equation}\label{unif-quantiles-bias}
{\rm bias}_{{\rm IND}}^{f}(f|X, G) := \int_0^1 | F_{Z_0} \circ F_{Z}^{[-1]}(q) - F_{Z_1} \circ F_Z^{[-1]}(q) | \, dq.
\end{equation}

The values in the integrand on the right-hand side of \eqref{unif-quantiles-bias} are weighted uniformly. In cases where they are weighted according to some probability distribution $\nu(dq)=\rho_{\nu}(q)dq$, one obtains a variation of \eqref{unif-quantiles-bias} that reads:
\begin{equation*}\label{nonunif-quantiles-bias}
{\rm bias}_{{\rm IND}}^{f,\nu}(f|X, G) := \int_0^1 | F_{Z_0} \circ F_{Z}^{[-1]}(q) - F_{Z_1} \circ F_Z^{[-1]}(q) | \, \nu(dq) = \int_{\mathcal{S}} | F_{Z_0}(t) - F_{Z_1}(t) | \, \rho_{\nu}(F_Z(t)) \, dt,
\end{equation*}
provided that $P_Z$ is atomless and $\mathcal{S}={\rm supp}(P_Z)$ is connected.
\end{remark}

\section{Estimators for distribution-invariant bias metrics}\label{sec::invar_metrics}

In Section \ref{sec::biasapprox}, we presented three estimators of \eqref{relax_bias}, the relaxed bias metric of \eqref{gen_bias}, where $h(\cdot)\geq 0$ is Lipschitz continuous and the probability measure $\mu \in \mathscr{P}(\RR)$ is fixed. The first estimator  is based on Monte Carlo (MC), the second on discretization, and the third is the energy estimator (for the non-relaxed metric). In Section \ref{sec::methods_perturb}, we discussed our approach to bias mitigation, which incorporates a gradient-based methodology -- built upon the aforementioned estimators -- to optimize over the family \eqref{linearfam}.

Gradient-based methods are invaluable in practical settings, as they are advantageous over derivative-free methods such as grid search and Bayesian optimization, especially when applied to high-dimensional parameter spaces. Access to derivatives enables an efficient and stable optimization, which is critical in an industry setting -- such as in finance and healthcare -- where computational costs and accuracy are of paramount importance.

However, when $\mu$ is fixed, the bias metrics \eqref{gen_bias} and \eqref{relax_bias} generally change under monotonic transformations of the model scores. In  applications where fairness in rank ordering is more meaningful, a distribution-invariant approach may be desired. To address this, \cite{Becker2024} proposed a modification to the Wasserstein-based bias \eqref{scorebias}  that removes its dependence on the model score distribution. Specifically, they consider the following distribution-invariant bias metric: 
\begin{equation}\label{app::inv_bias}
{\rm bias}_{\mathrm{IND}}^{f}(f|X, G) := \int |F_{0}(t)-F_{1}(t)|\cdot p_{f(X)}(t) \, dt 
\end{equation}
where $F_k$ is the CDF of $f(X)|G=k$, $k \in \{0,1\}$ and $p_{f(X)}$ is the density of the model scores; see also \cite{Vogel2021}. 

Motivated by this, we define the distribution-invariant analog of \eqref{gen_bias} by setting $\mu=P_{f(X)}$:
\begin{equation}\label{gen_bias_inv}
Bias^{(h)}_{f}(f|X,G):=\int h(F_0(t)-F_1(t)) \, P_{f(X)}(dt).
\end{equation}

In what follows, we extend the discussion in Section \ref{sec::biasapprox} by introducing estimators for the relaxed version of the invariant metric in \eqref{gen_bias_inv}:
\begin{equation}\label{gen_bias_inv_relax}
Bias^{(h)}_{f,s}(f|X,G):=\int h(F^{(s)}_0(t)-F^{(s)}_1(t)) \, P_{f(X)}(dt).
\end{equation}
where $F^{(s)}_k$ is the relaxed CDF defined by \eqref{relax_cdf}. Since these estimators can be used in the optimization problem \eqref{minfront} to evaluate the bias term $\B(\theta)$, we investigate their differentiability with respect to the parameter $\theta$. Below we examine each estimator separately.


\paragraph{\bf Estimator 1: Monte Carlo.} Here, we derive an analogue of the MC estimator \eqref{full-mc-est} of the relaxed bias metric \eqref{relax_bias} in the case where $\mu=P_{f(X)}$.

As before, let $f=f(\cdot;\theta)$ denote a predictive model, parameterized by $\theta$ and let $F_k(\cdot;\theta)$ be the CDF of $P_{f(X;\theta)|G=k}$, $k\in\{0,1\}$. Let $F^{(s)}_k(\cdot;\theta)$ denote the relaxed CDF defined by \eqref{relax_cdf}, and set $B_s(\cdot;\theta) = F^{(s)}_0(\cdot;\theta)-F^{(s)}_1(\cdot;\theta)$. We assume that $\theta \mapsto f(\cdot,\theta)$ is differentiable.

Let $D_k = \{ x_k^{(1)},\dots,x_k^{(m_k)} \}$ be i.i.d. samples from $P_{X|G=k}$, for $k=\{0,1\}$, and $D_{f(\cdot;\theta)} = \{ t_{\theta}^1,\dots,t_{\theta}^T \}$ are i.i.d. samples from $\mu_{\theta}:=P_{f(X;\theta)}$. For the latter, we can obtain the sample $t_{\theta}^i$ by first sampling independently $x^{(i)} \sim P_X$ and setting $t_{\theta}^i = f(x^{(i)};\theta)$.

Then, recalling \eqref{relax_bias}, we obtain the following estimator for \eqref{gen_bias_inv_relax}:
\begin{equation*}
 \int h \big( B_s(t;\theta) \big) \mu_{\theta}(dt) \approx \frac{1}{T}\sum_{j=1}^T h\big( \hat{B}_s(t_{\theta}^j;\theta) \big) = \frac{1}{T}\sum_{j=1}^T h\big( \hat{B}_s(f(x^{(j)};\theta);\theta) \big),
\end{equation*}
where $\hat{B}_s(t;\theta)$ is defined by \eqref{Bs_estimator}.

Since $f(\cdot;\theta)$ is linear in $\theta$, it is differentiable with respect to $\theta$, and in turn
\begin{equation}\label{full-mc-est-inv}
 \hat{B}_s(f(x^{(j)};\theta);\theta) = \frac{1}{m_1} \sum_{i=1}^{m_1} r_s(f(x^{(i)}_1;\theta)-f(x^{(j)};\theta)) - \frac{1}{m_0} \sum_{i=1}^{m_0} r_s(f(x^{(i)}_0;\theta)-f(x^{(j)};\theta))
\end{equation}
is also differentiable with respect to $\theta$ provided $r_s$ is chosen appropriately.

Finally, the above MC estimator for the relaxed bias metric is differentiable with respect to $\theta$ provided that both $h$ and $r_s$ are differentiable. Note that in practice this is not a concern: since $h$ and $r_s$ are assumed to be Lipschitz continuous, they are differentiable almost everywhere in their respective domains. Therefore, gradient-based methods such as stochastic gradient descent remain applicable and are expected to converge to an optimum.


\paragraph{\bf Estimator 2: Threshold-discrete.} Next, we derive an analogue of the threshold-discrete estimator \eqref{int_approx_est} of the relaxed bias metric \eqref{relax_bias} in the case where $\mu=P_{f(X)}$ has a density.

Let $f(\cdot;\theta), \ D_k,\ D_{f(\cdot;\theta)}$ be as above, and let  $\rho:=\rho(\cdot;\theta)$ denote the density of $\mu_{\theta}=P_{f(X;\theta)}$. For simplicity of exposition, we assume that $f=f(\cdot;\theta)$ is a classification score taking values in $[0,1]$, and that $\rho(\cdot;\theta)$ is a Lipschitz continuous function supported on $[0,1]$.

First,  given a uniform partition $\mathcal{P}_T := \{t_0=0 < t_1 < \dots < t_T=1\}$  of $[0,1]$, consider the estimator \eqref{int_approx_est} of the relaxed bias metric \eqref{relax_bias}, where $\mu$ is replaced with $\mu_{\theta}$, which is given by 
\begin{equation*} 
\int_0^1 h \big( B_s(t;\theta) \big) \mu_{\theta}(dt) = \int_0^1 h \big( B_s(t;\theta) \big) \rho_{\theta}(t) dt \approx \Big(\frac{1}{T}\sum_{j=1}^T h(\hat{B}_s(t^{(j)};\theta)) \rho(t_j;\theta)  \Big).
\end{equation*}

The only issue with using the above estimator in this case is that $\rho(\cdot;\theta)$, in general, does not have an explicit formulation. To circumvent this, we apply kernel density estimation \cite{silverman1986KDE} with a smooth kernel $K(\cdot)$ and bandwidth parameter $h>0$, which controls the smoothness of the estimated density. This allows us to estimate  $\rho$ as follows:
\begin{equation*}
 \hat{\rho}(t;\theta,h,N):= \frac{1}{Nh}\sum_{k=1}^N K\left( \frac{t-f(\tilde{x}^{(k)};\theta)}{h} \right),
\end{equation*}
where the samples $\{ \tilde{x}^{(1)},\dots, \tilde{x}^{(N)}\}$ are i.i.d from $P_X$, independent of $D_{f(\cdot;\theta)}$, and $K(\cdot)$ is a smooth kernel. This, in turn, allows us to approximate the relaxed bias metric \eqref{gen_bias_inv_relax} by
\begin{equation*}
 \int_0^1 h \big( B_s(t;\theta) \big) \mu_{\theta}(dt) \approx \frac{1}{T}\sum_{j=1}^T h \big( \hat{B}_s(t_j;\theta) \big) \bigg( \frac{1}{Nh}\sum_{k=1}^N K\left( \frac{t_j-f(\tilde{x}^{(k)};\theta)}{h} \right) \bigg)
\end{equation*}
where $\hat{B}_s(t;\theta)$ is defined by \eqref{Bs_estimator}. This approximation can then be used in any gradient-based approach. 


\paragraph{\bf Estimator 3: Energy.} Finally, consider the estimator \eqref{E-statistic-mu} of the bias metric in \eqref{energy_bias_mu}, where $h(\cdot)=2|\cdot|^2$, $\mu$ is replaced with $\mu_{\theta}$, and $f(\cdot;\theta)$, $D_k$, and $D_{f(\cdot;\theta)}$ are as defined above. 

As above, $F_{\mu_{\theta}}$ does not generally admit an explicit formulation. To address this, we can replace $F_{\mu_{\theta}}$ in the estimator \eqref{E-statistic-mu} with an estimate of its relaxed version:
\begin{equation*}
 \hat{F}^{(s)}_{\mu_{\theta}}(t) := 1-\frac{1}{T} \sum_{l=1}^{T} r_s(f(x^{(l)};\theta)-t), \quad f(x^{(l)};\theta) \in D_{f(\cdot;\theta)}.
\end{equation*}
We then obtain the following estimator:
\begin{equation*}
 2\int B_s(t;\theta)^2 \mu_{\theta}(dt) \approx\frac{2}{m_0 m_1} \sum_{i=1}^{m_0}\sum_{j=1}^{m_1} |\hat{F}^{(s)}_{\mu_{\theta}}(z_0^{(i)})-\hat{F}^{(s)}_{\mu_{\theta}}(z_1^{(j)})| -  \sum_{k \in \{0,1\}}\frac{1}{m_k^2} \sum_{i=1}^{m_k}\sum_{j=1}^{m_k} |\hat{F}^{(s)}_{\mu_{\theta}}(z_k^{(i)})-\hat{F}^{(s)}_{\mu_{\theta}}(z_k^{(j)})|,
\end{equation*}
where $z_k^{(i)}=f(x^{(i)}_k;\theta)$, $x^{(i)}_k \in D_k$, and $k \in\{0,1\}$.

Gradient-based methods are also applicable for this estimator as well, since the absolute value as a function is differentiable almost everywhere.

\end{appendices}



\section*{Acknowledgements}

The authors thank Alex Lin (Lead Research Scientist, Discover) for his valuable comments and editorial suggestions that aided us in writing this article.  We also thank Arjun Ravi Kannan (Director, Modeling, Discover) and Stoyan Vlaikov (VP, Data Science, Discover) for their helpful business and compliance insights.

\section*{Authors' contributions} 

Ryan Franks designed the high-level gradient-descent-based approach for optimizing encoder weights in post-processed models, developed the modified explainable variant of the optimal transport projection used for comparison, and proposed the use of energy distances and related unbiased estimators for bias mitigation. He also designed and implemented all experiments on public datasets and drafted the manuscript.

Alexey Miroshnikov proposed the conceptual framework for bias mitigation in ML models through post-processing methods optimized with a custom loss penalized by distribution-based fairness metrics. He introduced new fairness metrics, analyzed their theoretical properties, and developed differentiable estimators for their stochastic-gradient-descent optimization. He also proposed the use of weak learners as encoders in fairness optimization and contributed to revising the manuscript.

Kostas Kotsiopoulos performed the asymptotic analysis of differentiable estimators for non-invariant fairness metrics and provided technical proofs. He also proposed consistent estimators for a broad class of distribution-invariant fairness metrics and contributed to revising and editing the manuscript.

\bibliographystyle{unsrtnat}
\bibliography{ref}

\end{document}